\documentclass{article}

\PassOptionsToPackage{numbers, sort&compress}{natbib}

\usepackage{natbib}
\usepackage[final]{neurips_2022}

\usepackage[utf8]{inputenc} %
\usepackage[T1]{fontenc}    %
\usepackage{hyperref}       %
\usepackage{url}            %
\usepackage{booktabs}       %
\usepackage{amsfonts}       %
\usepackage{nicefrac}       %
\usepackage{microtype}      %
\usepackage{enumitem}

\usepackage{amssymb}
\usepackage{amsmath}
\usepackage{amsthm}
\usepackage{mathtools}
\usepackage{thmtools,thm-restate}
\usepackage{bm}

\usepackage{geometry}
\usepackage[utf8]{inputenc}
\usepackage{dsfont}
\usepackage{microtype}
\usepackage{comment}
\usepackage[table]{xcolor}
\usepackage{hyperref}
\usepackage{rotating}
\usepackage[skins]{tcolorbox}
\usepackage{amssymb}
\usepackage{amsmath}
\usepackage{amsthm}
\usepackage{mathtools}
\usepackage{thmtools}
\usepackage{bm}

\usepackage{geometry}
\usepackage[utf8]{inputenc}
\usepackage{dsfont}
\usepackage{microtype}
\usepackage{comment}
\usepackage[table]{xcolor}
\usepackage{hyperref}
\usepackage{rotating}
\usepackage[skins]{tcolorbox}
\usepackage{algorithm}
\usepackage{algpseudocode}
\usepackage{booktabs}
\newcommand{\op}{\text{op}}
\usepackage{tikz}
\usepackage{circuitikz}
\usetikzlibrary{shapes,backgrounds}
\usetikzlibrary{arrows.meta}
\usetikzlibrary{shapes.geometric}
\usepackage{caption}
\usepackage{subcaption}

\newcommand{\R}{\ensuremath{\mathbb{R}}}
\newcommand{\N}{\ensuremath{\mathbb{N}}}
\renewcommand{\d}{\,\mathrm{d}}
\newcommand{\e}{\,\mathrm{e}}

\newcommand{\cdf}{\mathcal{F}}
\newcommand{\pdf}{f}
\newcommand{\hcdf}{\hat{\mathcal{F}}}
\newcommand{\reg}{R_{\text{F}}}

\newcommand{\uniform}{\mathcal{U}}

\newcommand{\xmax}{L}
\newcommand{\Hminus}{\mathcal{H}^{-}}
\newcommand{\eigmin}{\lambda_{\min}}
\newcommand{\eigminplus}{\lambda_{\min}^{+}}

\newcommand{\instregret}{R_{\text{{\rm inst}}}}

\newcommand{\indic}[1]{\mathds{1}\left\{ #1 \right\}}

\newcommand{\tz}{\tilde t}

\newcommand{\isrand}{\mathcal{R}}

\newcommand{\E}{\mathbb{E}}
\renewcommand{\Pr}{\mathbb{P}}
\newcommand{\noisevar}{R}

\DeclarePairedDelimiterX{\scal}[2]{\langle}{\rangle}{#1, #2}
\DeclarePairedDelimiter{\norm}{\lVert}{\rVert}
\DeclarePairedDelimiter{\abs}{\lvert}{\rvert}
\DeclarePairedDelimiter\ceil{\lceil}{\rceil}
\DeclarePairedDelimiter\floor{\lfloor}{\rfloor}

\DeclareMathOperator{\Ker}{Ker}

\DeclareMathOperator*{\argmax}{\ensuremath{argmax}}

\declaretheorem[name=Theorem,refname=Thm.]{theorem}
\declaretheorem[name=Lemma,sibling=theorem,numberwithin=section]{lemma}

\declaretheorem[name=Proposition,refname=Prop.,sibling=theorem,numberwithin=section]{proposition}
\declaretheorem[name=Remark,numberwithin=section]{remark}

\declaretheorem[name=Corollary,refname=Cor.,sibling=theorem,numberwithin=section]{corollary}
\declaretheorem[name=Definition,refname=Def.,numberwithin=section]{definition}

\newtheorem{assumption}{Assumption}

\declaretheorem[name=Example]{example}
\usepackage{etoolbox}
\AfterEndEnvironment{restatable}{\noindent\ignorespacesafterend}
\AfterEndEnvironment{theorem}{\noindent\ignorespacesafterend}
\AfterEndEnvironment{remark}{\noindent\ignorespacesafterend}
\AfterEndEnvironment{example}{\noindent\ignorespacesafterend}
\AfterEndEnvironment{assumption}{\noindent\ignorespacesafterend}
\AfterEndEnvironment{lemma}{\noindent\ignorespacesafterend}

\usepackage{bbm}
\usepackage{cleveref}
\crefname{assumption}{Assumption}{Assumptions}
\crefname{equation}{}{}
\crefname{figure}{Fig.}{Fig.}
\crefname{table}{Table}{Tables}
\crefname{section}{Sec.}{Sec.}
\crefname{theorem}{Thm.}{Thm.}
\crefname{fact}{Fact}{Facts}
\crefname{lemma}{Lemma}{Lemmas}
\crefname{corollary}{Cor.}{Cor.}
\crefname{example}{Example}{Examples}
\crefname{remark}{Remark}{Remarks}
\crefname{algorithm}{Alg.}{Algorithms}
\crefname{appendix}{Appendix}{Appendices}

\usepackage[textwidth=2.0cm, textsize=tiny, disable]{todonotes} %

\newcommand{\MP}[2][]{\todo[color=orange!20,#1]{{\bf MP:} #2}}
\newcommand{\RG}[2][]{\todo[color=blue!20,#1]{{\bf RG:} #2}}

\newcommand{\LC}[2][]{\todo[color=purple!20,#1]{{\bf LC:} #2}}

\title{Group Meritocratic Fairness in Linear Contextual Bandits}

\author{
	Riccardo Grazzi\textsuperscript{1,2}\thanks{ \texttt{riccardo.grazzi@iit.it}} \ ,
	Arya Akhavan\textsuperscript{1,3}, John Isak Texas Falk\textsuperscript{1,2}, \\
	\textbf{Leonardo Cella\textsuperscript{1}, \
	Massimiliano Pontil\textsuperscript{1,2}} \\
	\\
	\textsuperscript{1}CSML, Istituto  Italiano di Tecnologia, Genoa, Italy\\
	\textsuperscript{2}Dept. of Computer Science, University College London, UK\\
	\textsuperscript{3}CREST, ENSAE, Institut Polytechnique de Paris, France\\
}

\begin{document}

\maketitle

\begin{abstract}
We study the linear contextual bandit problem where an agent has to select one candidate from a pool and each candidate belongs to a sensitive group. 
In this setting, candidates' rewards may not be directly comparable between groups, for example when the agent is an employer hiring candidates from different ethnic groups and some groups have a lower reward due to discriminatory bias and/or social injustice.
We propose a notion of fairness that states that the agent's policy is fair when it selects a candidate with highest relative rank, 
which measures how good the reward is when compared to candidates from the same group.
This is a very strong notion of fairness, since the relative rank is not directly observed by the agent
and depends on the underlying reward model and on the distribution of rewards.
Thus we study the problem of learning a policy which approximates a fair policy under the condition that the contexts are independent between groups and the distribution of rewards of each group is absolutely continuous. In particular, we design a greedy policy which at each round constructs a ridge regression estimate from the observed context-reward pairs, and then computes an estimate of the relative rank of each candidate using the empirical cumulative distribution function.
We prove that, despite its simplicity and the lack of an initial exploration phase, the greedy policy achieves, up to log factors and with high probability, a fair pseudo-regret of order $\sqrt{dT}$ after $T$ rounds, where $d$ is the dimension of the context vectors. 
The policy also satisfies demographic parity at each round when averaged over all possible information available before the selection. Finally, we use simulated settings and experiments on the US census data to show that our policy achieves sub-linear fair pseudo-regret also in practice.
\end{abstract}

\section{Introduction}
Consider a sequential decision making problem where at each round an employer has to select one candidate from a pool to hire for a job. The employer does not know how well a candidate will perform if hired, but they %
can learn it over time by measuring the performance of previously selected  similar candidates. This scenario can be formalized as a (linear) contextual bandit problem (see \citep{lattimore2020bandit,auer2002using,chu2011contextual} and references therein), where each candidate is represented by a context vector, and after the employer (or agent) chooses a candidate, it receives a reward, i.e.\@ a scalar value measuring the true performance of the candidate, which depends (linearly) on the context. 

In the above framework, the typical objective is to find a policy for the employer to select candidates with the highest rewards \citep{lattimore2020bandit,abbasi2011improved,auer2002nonstochastic,auer2002using}. However, in some important scenarios this objective may not be appropriate; if candidates belong to different sensitive groups 
(e.g.\@ based on ethnicity, gender, etc.) the resulting policy might discriminate or even exclude some groups completely in the selection process.
This may happen when some groups have lower expected reward than others, e.g.\@ because they acquired less skills due lower
financial support.
Another example 
arises when each candidate in the pool, if selected, will perform a different kind of job, and the associated reward is job-specific. For instance,
if the employer is a university and each candidate is a researcher in a different discipline, then
the rewards associated to different disciplines will be substantially different and incomparable, e.g.\@ citation counts vary greatly among different subjects; see \citep{kearns2017meritocratic} for a discussion. 
In both of the above scenarios, it is unfair to directly compare rewards of candidates belonging to different groups. 

A simple way to deal with this issue would be to select the candidate to hire uniformly at random. This policy satisfies a notion of fairness called \textit{demographic parity} (see \cite{calders2009,mehrabi2021} and references therein), which requires the probability of selecting a candidate from a given group to be equal for all groups. 
However, as is apparent, this approach completely ignores the employer's goal of selecting good candidates and is also unfair to candidates who spent effort acquiring credentials for the job. In this work, we provide a fair way of comparing candidates from different groups via the \textit{relative rank}, that is the cumulative distribution function (CDF) value of the reward of the candidate where the distribution is that of the rewards of the candidate's group. We call a policy \textit{group meritocratic fair} (GMF) if it always selects a candidate with the highest relative rank. Such a policy is meritocratic but only in terms of the within-group performance. 
A closely related idea has been introduced in \cite{kearns2017meritocratic} for settings where the candidates' rewards are available before the selection, while we are not aware of a similar notion in the multi-armed bandits literature.

A GMF policy requires the knowledge of the relative rank of each candidate 
which is not directly observed by the agent
and depends on the underlying reward model and on the distributions of rewards.
Moreover, to estimate the relative rank from the observed rewards and contexts it is necessary to learn the CDF of the rewards of each group, which adds a challenge to the standard linear contextual bandit framework where only the linear relation between contexts and rewards has to be learned.
Due to this, a learned policy cannot be GMF at all rounds, thus we study the problem of learning a policy which minimizes the \textit{fair regret}, that is the cumulative difference between the relative rank of the candidate chosen by a GMF policy and the candidate chosen by a learning policy.

For this purpose, we design a greedy policy, which at each round uses the following two-stage strategy.
Firstly, it constructs a ridge regression estimate which maps contexts to rewards. Secondly, it computes an estimate of the relative rank of each candidate using the empirical CDF of the estimated rewards.
We show that the proposed policy achieves, under some reasonable assumptions and after $T$ rounds, $\tilde{O}(K^3 + \sqrt{dT})$ fair regret with high probability, where $d$ is the dimension of the context vectors and $K$ is the number of candidates in the pool.
Notably, our policy does not require an initial exploration phase and satisfies demographic parity at each round when averaged over all possible random draws of the information avaliable to the agent before the decision, i.e.\@ current contexts and previously received contexts, actions and rewards.

\textbf{Contributions and Organization.} After a review of previous work in \Cref{sec:related}, we introduce the learning problem and the proposed fairness notion in \Cref{sec:problem}. To simplify the exposition, we assume that each arm corresponds to a sensitive group. In \cref{sec:policy} we propose a greedy policy which jointly learns the underlying regression model and the CDF of each group.
We derive a $\tilde{O}(\sqrt{dT})$ regret bound for our policy in \Cref{sec:regret}. In \Cref{se:simu} we present an illustrative simulation experiment with diverse reward distributions. In \Cref{se:multiplemain}, we extend our policy and results to the case where candidates  from the same arm can belong to different groups and show the efficacy of our approach with an experiment on the US census data where the sensitive group (ethnicity) is drawn at random together with the context. We draw conclusions in \Cref{se:conclusions}. Code at \url{https://github.com/CSML-IIT-UCL/GMFbandits}

\textbf{Notation.} We use $\scal{\cdot}{\cdot}$ for the scalar product. For $K \in \N$ we have $[K] = \{1,\dots, K\}$. Let $\psi$ be a scalar random variable, for each $a \in [K]$ and $\mu^{*} \in \R^d$, we denote with $\cdf_{\psi}$, $\cdf_{a}$ the CDF of $\psi$ and $\scal{\mu^*}{X_a}$  respectively. If $X$ is a continuous random vector with values in $\R^d$, we denote with $\pdf_{X}: \R^d \to [0, \infty)$ its probability density function. 
For any $s \in \N$, we denote $\mathbbm{I}_s$ as the $s\times s$ identity matrix. For a random variable $Y \in \R^s$, we call $Y$ an absolutely continuous random variable, if its distribution is an absolutely continuous measure with respect to the Lebesgue measure on $\R^s$. For a positive semi-definite matrix $D$, we denote $\eigmin(D)$ and $\eigminplus(D)$  the smallest eigenvalue of $D$, and the smallest non-zero eigenvalue of $D$ respectively.
${\rm Supp}(X)$ indicates the support of a random variable $X$. We also denote with $\uniform[S]$, the uniform distribution over the set $S$.

\section{Related Works}
\label{sec:related}

In recent years algorithmic fairness has received a lot of attention, becoming a large area of machine learning research.  The potential for learning algorithms to amplify pre-existing bias and cause harm to human beings has triggered researchers to study solutions to mitigate or remove unfairness of the learned predictor, see \cite{barocas-hardt-narayanan,calmon2017optimized,chierichetti2017fair,Donini_Oneto_Ben-David_Taylor_Pontil18,dwork2018decoupled,hardt2016equality,jabbari2016fair,joseph2016fairness,kilbertus2017avoiding,kusner2017counterfactual,lum2016statistical,yao2017beyond,zafar2017fairness,zemel2013learning,zliobaite2015relation} and references therein.
Fairness in sequential decision problems (see \cite{zhang2021fairness} for a survey) is usually divided into two categories: group fairness (GF)  and individual fairness. We give an overview of these notions below.

GF requires some statistical measure to be (approximately) equal across different sensitive groups. A prominent example relevant to this work is \textit{demographic parity}, which requires that the probability that the policy selects a candidate from a given group should be the same for all groups. A similar notion is used by \cite{chen2020fair,patil2020achieving}, where the probability that the policy selects a candidate has to always be greater than a given threshold for all candidates.
\cite{li2019combinatorial} impose a weaker requirement concerning the expected fraction of candidates selected from each group.
Other examples of GF in sequential decision problems are \textit{equal opportunity} \citep{bechavod2019equal} and \textit{equalized odds} \citep{blum2018preserving}. 
Under some assumptions on the distributions of the contexts, our GMF and greedy policies satisfy variants of demographic parity at each round.

Individual fairness can be divided in two categories: fairness through awareness (FA) \citep{liu2017calibrated,wang2021fairness} and meritocratic fairness (MF) \citep{joseph2016fairness,joseph2018meritocratic}. 
FA is based on the idea that similar individuals should be treated similarly and is designed to avoid ``winner takes all'' scenarios where some individuals cannot be selected when they have a lower reward than others in the pool, even if the difference between rewards is very small. For example, \cite{wang2021fairness} propose a policy where the probability of selecting a context over another is lower when the context has a lower reward, but is never zero.  
MF instead requires that less qualified individual should not be favored over more qualified ones, which could happen during the learning process. For example \cite{joseph2016fairness} proposes an algorithm where the policy selects the arm uniformly at random among the best arms with overlapping confidence intervals.
This guarantees meritocratic fairness at each round but comes at a cost in terms of regret.
\RG{Can we be more specific on this cost?}

Our definition of fairness falls between group and meritocratic fairness. It is meritocratic because it states that a candidate with a worse relative rank than another should never be selected. It is also based on groups since the relative ranks directly depend on the distribution of rewards of each group. A similar idea of fairness based on relative rank has been introduced in \cite{kearns2017meritocratic}, which study the problem of selecting candidates from different groups based on their scalar-valued score when the scores between groups are incomparable (e.g.\@ number of citations in different research areas). Contrary to our work, where the (noisy) rewards are observed only for the selected candidates,  in \cite{kearns2017meritocratic} the noiseless scores for all candidates can be accessed before the selection. This difference makes the estimation of the relative rank simpler in \cite{kearns2017meritocratic}, as the rewards CDFs can be estimated more efficiently. 

\RG{TODO (Maybe):  pargaph about how fairness affects future contexts and rewards etc.. see the survey (P2 problems)?}
\RG{We do not have any technical discussion here for now, it may seem a bit shallow. Although we are introducing a new setting so it may be fine.}
\LC{I agree. Up to now I think we are missing a paragraph specifying the difficulties met and the novelty. Maybe we can enrich the contribution paragraph with such details.}

\section{Group Meritocratic Fairness in Linear Contextual Bandits}
\label{sec:problem}

We consider the linear contextual bandit  setup \citep{auer2002using}
where at each round $t \in [T]$, an agent receives a set of feature vectors $\{X_{t, a}\}_{a=1}^{K}$ with $X_{t, a} \subset \R^d$ sampled from the environment, one for each arm $a \in [K]$. We assume that context (or candidate) $X_{t,a}$ has an associated reward $\scal{\mu^*}{X_{t,a}}$ where $\mu^* \in \mathbb{R}^d$ is unknown to the agent. After the agent selects the arm $a_t$, it receives the noisy reward 
equal to $r_{t,a_t} = \scal{\mu^*}{X_{t,a_t}} + \eta_t$, where $\eta_{t}$ is some scalar noise (formally specified later).
In addition, we assume that each arm represents a fixed sensitive group (e.g.\@ based on ethnicity, gender, etc.). 
The latter assumption simplifies the presentation but implies that at each round the agent receives exactly one candidate for each group. This can be too restrictive e.g.\@ when candidates are sampled i.i.d. together with their group and/or some groups are minorities. However, our results can be easily adapted to more realistic settings without such assumption, as we show in \Cref{se:multiplemain} and more rigorously in \Cref{se:multtiple}. Excluding these sections, we use arm and group interchangeably in all that follows.

Usually, the goal of the agent is to maximise the expected cumulative reward $\sum_{t=1}^{T} \scal{\mu^*}{X_{t, a_t}}$. Since as we previously explained, this objective might be unfair to some of the sensitive groups, we instead use a different kind of reward which measures the relative performance of a candidate compared to others of the same arm/group.
First, we additionally assume, for each group $a$, that $\{X_ {t,a}\}_{t=1}^{T}$ are i.i.d and have the same distribution of $X_a$, which we define to be a random variable with unknown distribution. We call the distribution of $\scal{\mu^*}{X_a}$ the reward distribution of arm $a$ and denote with $\cdf_a$ its CDF, i.e.\@ $\cdf_a(r) =  \Pr(\scal{\mu^*}{X_a} \leq r)$ for every $r \in \R$. Then, we introduce the \textit{ relative rank} of candidate $X_{t,a}$ as $\cdf_a(\scal{\mu^*}{X_{t,a}})$, that is the probability that a sample from 
the reward distribution of 
arm $a$ is lower than the  reward of $X_{t,a}$. We argue that the relative rank, allows to have a fair way of comparing candidates from different groups and introduce the following fairness definition.

\begin{definition}[Group Meritocratic Fairness]\label{def:gmf} A policy $\{a_t^*\}_{t=1}^{\infty}$ is group meritocratic fair (GMF) if for all $t \in \N, a \in [K]$
\begin{equation*}
     \cdf_{a_t^*}(\scal{\mu^*}{X_{t, a_t^*}}) \geq \cdf_{a}(\scal{\mu^*}{X_{t, a}})\enspace.
\end{equation*}
\end{definition}
A GMF policy chooses candidates with the highest reward compared to candidates from the same group.
This is a strong definition of fairness which is impossible to satisfy at each round for a learned policy. As in standard linear contextual bandits, $\mu^*$ is unknown and must be learned. In this setting however, we have the additional challenge of learning the CDF for the rewards of each arm, $\cdf_a$. Thus, we will focus on how to learn a GMF policy by introducing the following regret definition.

\begin{definition}[Fair Pseudo-Regret]
Let $T \in \N$, $\{a_t\}_{t=1}^{T}$ be the evaluated policy and $\{a_t^*\}_{t=1}^{T}$ be a GMF policy. Then we denote by (cumulative) \textit{fair pseudo-regret} the quantity
\begin{equation}\label{eq:fairregret}
    \reg(T) := \sum_{t=1}^T  \cdf_{a_t^*}(\scal{\mu^*}{X_{t, a_t^*}}) -  \cdf_{a_t}(\scal{\mu^*}{X_{t, a_t}})\enspace.
\end{equation}
\end{definition}
The goal of the learned policy will be to minimize the fair pseudo-regret, since a policy with sublinear fair pseudo-regret will get closer and closer to a GMF fair policy over time. 
\begin{remark}\label{rem:standard}
The fair pseudo-regret resembles the standard pseudo-regret defined as
\begin{equation*}
     R(T) := \sum_{t=1}^T  \scal{\mu^*}{X_{t, a_t^{\text{{\rm opt}}}}} -  \scal{\mu^*}{X_{t, a_t}} \quad\text{with}\quad a_t^{{\rm opt}} \in  \argmax_{a\in [K]} \scal{\mu^*}{X_{t,a}}\enspace,
\end{equation*}
where rewards are replaced by relative ranks and $a_t^{\rm {opt}}$ by the GMF policy $a_t^*$. Furthermore, since the CDF restricted to the support is strictly increasing, when the reward distributions are the same for each arm, i.e.\@ $\cdf_a = \cdf_{a'}$ for all $a,a'\in [K]$, then  a policy minimizing the fair pseudo-regret also minimizes the standard pseudo-regret and vice versa.
This is not true in the general case, where fair and standard pseudo-regrets are often competing objectives. For example, when $\{\scal{\mu^*}{X_a}\}_{a=1}^{K}$ are independent and absolutely continuous and there exists $\hat a$ such that $\scal{\mu^*}{X_{\hat a}} > \scal{\mu^*}{X_a}$ for every $a \neq \hat a$, then for every $t$, $a^{\rm {opt}}_t = \hat a$, while as we will show in \Cref{pr:gmdm}, $a_t^*$ selects each arm with equal probability. Thus, with non-zero probability $a^{\rm {opt}}_t$ has a linear fair pseudo-regret while  $a_t^*$ has a linear standard pseudo-regret. Moreover, in \Cref{sec:linreg}, for $K=2$, we show that if $\scal{\mu^*}{X_1}$ and $\scal{\mu^*}{X_2}$ are independent, absolutely continuous, but not identically distributed, then the GMF policy has a linear standard regret and $\{a_t^{\text{opt}}\}_{t=1}^{\infty}$ has a linear fair regret with positive probability.
\end{remark}

Learning a GMF policy brings several challenges. The relative rank is not directly observed by the agent, which receives instead only the noisy reward. This implies that the agent has to estimate $\cdf_a$, which in general might not even be Lipschitz continuous. This is the main reason why we restrict our analysis to the case where the rewards $\{\scal{\mu^*}{X_a}\}_{a=1}^{K}$ are independent and absolutely continuous.  In particular, for any $t \geq 0$, let $\Hminus_t : = \cup_{i=1}^{t} \left\{\{X_{i, a}\}_{a=1}^K, r_{i, a_{i}}, a_{i}\right\}$ with $\Hminus_0 = \varnothing$ and $ \mathcal{H}_t := \Hminus_t \cup \{\{X_{t+1, a}\}_{a=1}^K\}$ be respectively the history and the information available for the decision at round $t+1$, then the following holds.

\begin{proposition}[GMF policy satisfies \textit{history-agnostic demographic parity}]\label{pr:gmdm}
Let $\{\scal{\mu^*}{X_a}\}_{a=1}^{K}$ be independent and absolutely continuous and for every $a\in[K], t \in \N$, let $X_{t,a}$ be an i.i.d. copy of $X_a$. Then for every $t \in \N$, $\{\cdf_a(\scal{\mu^*}{X_{t,a}})\}_{a=1}^K$ are i.i.d. uniform on $[0,1]$ and
\begin{equation}\label{eq:gmdm}
    \Pr (a_t^* = a \,|\, \Hminus_{t-1}) =
    \frac{1}{K} \qquad \forall a
    \in [K],
\end{equation}
for any GMF policy $\{a^*_t\}_{t=1}^{\infty}$. Note, the randomness lies exclusively in the current contexts $\{X_{t,a}\}_{a=1}^{K}$.
\end{proposition}
\vspace{-.3cm}
\begin{proof}
Let $\psi_a : = \cdf_a(\scal{\mu^*}{X_{t,a}})$.
From the assumptions $\{\psi_a\}_{a=1}^{K}$ are i.i.d random variables, independent from $\Hminus_{t-1}$, with uniform distribution on $[0,1]$ (see \cite[Theorem 2.1.10]{casella2021statistical}). Hence $\forall a_1, a_2 \in [K]$: $\Pr(\psi_{a_1} = \psi_{a_2}) = 0$,  $\Pr (a_t^* = a \,|\, \Hminus_{t-1}) = \Pr(a_t^* = a)$ and
\begin{equation*}
\Pr(a_t^* = a_1 ) = \Pr(\psi_{a_1} > \psi_{a'}, \, \forall a' \neq a_1) = \Pr(\psi_{a_2} > \psi_{a'}, \, \forall a' \neq a_2) =  \Pr (a_t^* = a_2) = 1/K\enspace.
\end{equation*}
\end{proof}
We call property \eqref{eq:gmdm} history-agnostic demographic parity since it states that, at each round, the policy selects all groups with equal probability regardless of the history. Recall that in our setup each arm corresponds to a sensitive group.
\Cref{pr:gmdm} ensures that a GMF policy will keep exploring regardless of the history. This fact plays a key role in the design of our policy, which is greedy without the need of an exploration phase.
\begin{remark}
Note that in the standard linear contextual bandit setting, the optimal policy $a_t^{\rm opt}$ does not necessarily satisfy \Cref{eq:gmdm} even when we assume that $\{\scal{\mu^*}{X_a}\}_{a=1}^{K}$ are independent and absolutely continuous. This is true since when the rewards of one arm are always lower than at least one of the other arms, that arm will never be selected by the optimal policy. 
\end{remark}

In the following, we state and  discuss the assumptions made for the analysis of our greedy policy.

\begin{assumption}\label{assump}
Let $\mu^* \in \R^d$ be the underlying reward model.
We assume that:
\begin{enumerate}[label={\rm (\roman*)}]
    \item\label{ass:noise} The noise random variable $\eta_t$ is zero mean $\noisevar$-subgaussian, conditioned on $\mathcal{H}_{t-1}$. 
    
    \item\label{ass:iid} Let $X_a$ be a random variable with values in $\R^d$ and such that $\norm{X_a}_2 \leq \xmax$ almost surely.  For any $a \in [K]$, $\{X_{i,a}\}_{i=1}^{T}$ are i.i.d. copies of $X_a$.
    
    \item\label{ass:indep} The random variables $\{X_a\}_{a=1}^{K}$ are mutually independent. 
    
    \item\label{ass:abscont} 
    For every $a \in [K]$, there exist $d_a \geq 1$, an absolutely continuous random variable $Y_a$ with values in $\R^{d_a}$ admitting a density $\pdf_{a}$,  $B_a \in \R^{d \times d_a}$ and $c_a \in \R^d$ such that $B_a^\top B_a = \mathbbm{I}_{d_a}$, 
    \begin{equation*}
        X_a = B_a Y_a + c_a \quad\text{and}\quad \mu^{*\top} B_a \neq 0\enspace.
    \end{equation*}
    \RG{can we relax this by saying that $c_a$ is a discrete random variable (possibly dependent on $Y_{t,a}$?) is it worth it? The form of the theory make it seem like we could just require the reward to be abs. cont.}
    
\end{enumerate}
\end{assumption}

\Cref{assump}\ref{ass:noise} is a standard assumption on the noise in stochastic bandits. \ref{assump}\ref{ass:iid} implies that the actions taken by the policy do not affect future contexts.  This is needed to allow the learning of the distribution of rewards for each group and is also used in \cite{li2019combinatorial,chen2020fair}. \ref{assump}\ref{ass:abscont} implies that $\scal{\mu^*}{X_a}$ is absolutely continuous and is satisfied when $X_a$ is absolutely continuous in a subspace of $\R^d$ which is not orthogonal to $\mu^*\,$\footnote{
E.g.~$X_a$ cannot be sum of random variables that are independent and absolutely continuous  in orthogonal subspaces of $\R^d$.}\RG{I'm unsure that ``$X_a$ is absolutely continuous in a subspace of $\R^d$''  is correct. What could we say instead?}.
This fact combined with \ref{assump}\ref{ass:indep} ensures that \Cref{pr:gmdm} holds. 
Assumptions~\ref{assump}\ref{ass:indep}-\ref{ass:abscont} are specific to our setting and a current limitation of the analysis. 
Notice however, that \ref{assump}\ref{ass:indep} is reasonable when the groups are sufficiently isolated, e.g.\@ each context is sourced from a different country/group, while assuming that the rewards $\scal{\mu^*}{X_a}$ are absolutely continuous is natural when the contexts contain continuous attributes. Furthermore \ref{assump}\ref{ass:abscont} allows $\mu^*$ to act differently on each group, similarly to the case when there is a different reward vector for each sensitive group. An example of this is showed in the simulation experiment in \Cref{se:simu}.

\section{The Fair-Greedy Policy}\label{sec:policy}

If \Cref{pr:gmdm} holds, then there is no arm with relative rank always strictly worse than the others and any learned policy with sub-linear fair pseudo-regret will select all arms with equal probability in the limit when the number of rounds goes to infinity.
Hence, using confidence intervals 
will not help in decreasing the probability that one arm is selected. Furthermore, estimating the relative ranks $\{\cdf_a(\scal{\mu^*}{X_{t,a}})\}_{a=1}^{K}$ is challenging, since they are not directly observed and using the past noisy rewards $\{r_{i,a_i}\}_{i=1}^{t-1}$ to construct the empirical CDF for each group, similarly to \cite{kearns2017meritocratic}, can be inaccurate due to the presence of noise.

For the reasons above, we propose the greedy approach in \Cref{alg:fairgreedy}, which  uses the following two-stage procedure at each round $t$. First it  assembles the previously selected contexts and corresponding rewards from iterate $1$ up to $\tz = \floor{(t-1)/2}$ (\cref{step:assemble}) in order to construct an estimate $\mu_{\tz}$ of $\mu^*$ (\cref{step:mu}), which is a noisy version of the ridge regression estimate. Secondly, for each arm $a$, our policy computes an estimate of the relative rank $\cdf_a(\scal{\mu^*}{X_{t,a}})$, namely $\hcdf_{t,a}(\scal{\mu_{\tz}}{X_{t,a}})$, which is the empirical CDF value of $\scal{\mu_{\tz}}{X_{t,a}}$ and is constructed using $\mu_{\tz}$ and the contexts from round $\tz+1$ up to $t$ (\cref{step:ecdf}). 
Lastly, it selects $a_t$ uniformly at random among the arms maximizing the relative rank estimate (\cref{step:at}).

\begin{algorithm}[ht!]
\caption{Fair-Greedy}\label{alg:fairgreedy}
\begin{algorithmic}[1]
\State \textbf{Requires} regularization parameter $\lambda > 0$ and noise magnitude $\rho \in (0, 1]\enspace.$
\For{$t=1 \dots T$}
\State\label{empty1} Receive contexts $\{X_{t,a}\}_{a=1}^{K}$
\State\label{step:assemble} Set $\tz = \floor{(t-1)/2}$,  $X_{1:\tz} =  (X_{1,a_1}, \dots, X_{\tz,a_{\tz}})^\top $,  $r_{1:\tz} = (r_{1,a_1}, \dots, r_{\tz,a_{\tz}})$.
\State\label{step:mu} \textbf{If} $\tz = 0$ set $\mu_{\tz} = 0$, \textbf{else} let $V_{\tz} := X_{1:\tz}^\top X_{1:\tz} + \lambda \mathbbm{I}_d$, generate
$\gamma_{\tz} \sim \mathcal{N}(0, \mathbbm{I}_d)$ and compute 
\begin{equation*}
    \mu_{\tz} := V_{\tz}^{-1}X_{1:\tz}^\top r_{1:\tz} + \frac{\rho}{d\sqrt{\tz}}\cdot\gamma_{\tz}\enspace.
\end{equation*}
\State\label{step:ecdf} For each $a \in [K]$ compute 
\begin{equation*}
\hcdf_{t, a} (\scal{\mu_{\tz}}{ X_{t,a}}) :=  (t-1-\tz)^{-1} \sum_{s=\tz+1}^{t-1} \indic{ \scal{\mu_{\tz}}{ X_{s,a}} \leq \scal{\mu_{\tz}}{X_{t,a}}}\enspace.
\end{equation*}

\State\label{step:at} Sample action
\begin{equation*}
  a_t \sim \mathcal{U}\big[\argmax_{a\in [K]} \hcdf_{t, a} (\scal{\mu_{\tz}}{X_{t,a}})\big]\enspace.
\end{equation*}
\State\label{empty2} Observe noisy reward $r_{t,a_t} = \scal{\mu}{X_{t,a_t}} + \eta_t$.

\EndFor
\end{algorithmic}
\end{algorithm}

Fair-Greedy has two hyperparameters $\lambda$ and $\rho$, although the latter can be set arbitrarily small without affecting the regret. Moreover, it is greedy as at each time $t$, it always selects from the arms the one with the highest currently estimated relative rank. However, contrary to standard greedy approaches in bandits, Fair-Greedy does not require an initial exploration phase because it naturally explores all arms, as the following lemma and remark show.

\begin{lemma}[Fair-Greedy satisfies \textit{information averaged demographic parity}]\label{DGP}
Let $a_t$ be the action taken by Fair-Greedy at time $t$ and let \Cref{assump} be satisfied. Then, for all $t\geq 1$ we have 
\begin{align}\label{eq:iagm}
    \Pr(a_t = a) &= \frac{1}{K}\enspace.
\end{align}

\end{lemma}
\begin{proof}[Proof sketch (proof in \Cref{se:proofDGP})]
    The noise term in $\mu_{\tz}$ ensures that $\mu_{\tz}$ is absolutely continuous and hence $\mu_{\tz}^{\top}B_a \neq 0$ almost surely. Combining this with \Cref{assump}\ref{ass:abscont} we obtain that $\scal{\mu_{\tz}}{ X_a}$ is also absolutely continuous (see \Cref{lm:prodabs}). Moreover, thanks to \Cref{assump}\ref{ass:iid}\ref{ass:indep} we can show that the random variables in $\{\hcdf_{t, a} (\scal{\mu_{\tz}}{X_{t,a}}) \}_{a=1}^{K}$ are i.i.d. when conditioned on $\mu_{\tz}$. Note that $a_t$ is sampled uniformly form the argmax of i.i.d. random variables, when conditioned on $\mu_{\tz}$, which implies $\Pr(a_t = a \,|\, \mu_{\tz}) = 1/K$. The statement follows by taking the expectation over $\mu_{\tz}$.
\end{proof}

\begin{remark}
   It is easy to verify (through \Cref{DGP}) that at any number of rounds $T$, the Fair-Greedy policy selects in expectation $T/K$ candidates from every group, i.e.\@ $\E\big[\sum_{t=1}^T\indic{a_t = a}\big] = \frac{T}{K}$ for every $a \in [K]$. This also holds for the GMF policy and the one selecting arms uniformly at random. 
\end{remark}

Since $\Pr(a_t = a) = \E_{\mathcal{H}_{t-1}}[\Pr(a_t = a \,|\, \mathcal{H}_{t-1}) ]$, with $\mathcal{H}_{t-1}$ being the information available to the policy before making a decision at round $t$, we call the property in \eqref{eq:iagm} information-averaged demographic parity, which is weaker than history-agnostic demographic parity (in \eqref{eq:gmdm}). 
However, our analysis still requires a lower bound on $\Pr(a_t = a \,|\, \Hminus_{t-1})$ which is presented in the next section.\RG{I think this paragraph can be improved}

\MP{Let's discuss the cost (i do not understand why you have $d$ in $K d(t-\tz)$ and whether it can be reduced . Also when you say update online it is maybe a bit fast}
\RG{To compute step 6 you need to use the new $mu_{\tz}$ and recompute all the scalar products which costs $K d (t-\tz)$, at least every other iteration. You also need, at every round, to then do $K(t-\tz)$ the comparisons an sum them up. I don't know how it can be less costly unless you assume the rewards have a specific distribution like truncated gaussian and you simply estimate its parameters.}
\RG{should we say how this can be alleviated in practice by recomputing the rewards it every once in a while. you can also order them to then have log(n) computation for the ECDF, which is what i do to compute the oracle}

\begin{remark}[Computational cost of Fair-Greedy]
Compared to common linear contextual bandits approaches based on ridge regression, 
\Cref{alg:fairgreedy} has an higher computational and memory cost which grow linearly with $t$. $\mu_{\tz}$ requires us to compute the product of  $V^{-1}_{\tz}$ and $X_{1:\tz}^\top r_{1:\tz}$, which can be stored using $d^2$ and $d$ values respectively and updated online (via sherman-morrison  \citep{hager1989updating}). However, \Cref{alg:fairgreedy} also requires, at each round $t$, to keep in memory $K(t -1- \tz)$ $d$-dimensional contexts and to compute the same number of scalar products to construct the  empirical CDF for all $K$ groups.
\end{remark}

\section{Regret Analysis}
\label{sec:regret}

In this section we present the analysis leading to the high probability $\tilde{O}(K^3 + \sqrt{dT})$ upper bound on the fair pseudo-regret of the greedy policy in \Cref{alg:fairgreedy}.
We start by showing two key properties of CDF functions in the following lemma (proof  in \Cref{se:proodcdfbound}). Recall that for a continuous random variable $Z$ we denote by $\pdf_{Z}$ the associated probability density function (PDF).

\begin{lemma}\label{lm:cdfbound}
Let \Cref{assump}\ref{ass:abscont} hold and set $\forall a \in [K]$, $Z_a := \scal{\mu^*}{X_a}$ so that $\cdf_a = \cdf_{Z_a}$ and
    $ M := \max_{a \in [K], z \in \R} \pdf_{Z_a}(z) < +\infty$ as the maximum PDF value of the rewards of all groups.
Then, the following two statements are true.
\begin{enumerate}[label={\rm (\roman*)}]
\item\label{lm:cdfbound_1} $\cdf_a$ is Lipschitz continuous for every $a \in [K]$, and in particular for any $r, r' \in \R$ we have
\begin{equation*}
    \sup_{a \in [K]}\abs{\cdf_a(r) - \cdf_a(r')} \leq M \abs{r-r'}\enspace.
\end{equation*}

\item\label{lm:cdfbound_2} For every $a \in [K]$, let $\mu \in \R^d$, $\tilde{Z}_a := \scal{\mu}{X_a}$. Then we have
\begin{align*}
  \sup_{a\in [K], r \in \R}\abs{\cdf_a(r) - \cdf_{\tilde{Z}_a}(r)} &\leq  2M\norm{\mu^{*} - \mu}\norm{x_{\max}}_{*}\enspace,
\end{align*}
for any norm $\norm{\cdot}$ with dual norm $\norm{\cdot}_*\enspace$, where $\norm{x_{\max}}_{*} := \sup_{x \in \cup_{a=1}^{K}{\rm Supp}(X_a)}\norm{x}_*\enspace$ and ${\rm Supp}(X_a)$ is the support of the random variable $X_a\enspace$.
\end{enumerate}
\end{lemma}

 \Cref{lm:cdfbound}\ref{lm:cdfbound_1} bounds the Lipschitz constant of $\cdf_a$ and its derivation is straightforward. \Cref{lm:cdfbound}\ref{lm:cdfbound_2} is needed since we only have access to an estimate of $\mu^*$, which will take the role of $\mu$. Its derivation is more subtle and could be of independent interest.
By using \Cref{lm:cdfbound} and the Dvoretzky–Kiefer–Wolfowitz-Massart (DKWM) inequality \cite{Dvoretzky_Kiefer_Wolfowitz56,DKWM} to bound the gap between CDF and empirical CDF,  we obtain the following result.

\begin{lemma}[Instant regret bound]
\label{lm:instantbound}
Let \Cref{assump}\ref{ass:iid}\ref{ass:abscont} hold and $a_t$ to be generated by \Cref{alg:fairgreedy}. Then with probability at least $1-\delta/4$, for all $t$ such that $3 \leq t \leq T$ we have
\begin{align*}
    \cdf_{a_t^*}(&\scal{\mu^*}{X_{t,a_t^*}}) -\cdf_{a_t}(\scal{\mu^*}{X_{t, a_t}}) \leq 6M\norm{\mu^* - \mu_{\tz}}_{V_{\tz}}\norm{x_{\max}}_{V^{-1}_{\tz}} + 2\sqrt{\frac{\log(8KT/\delta)}{t-1}}\enspace,
\end{align*}
where $\norm{x_{\max}}_{V^{-1}_{\tz}} := \sup_{x \in \cup_{a=1}^{K}{\rm Supp}(X_a)}\norm{x}_{V^{-1}_{\tz}}$.
\end{lemma}
\begin{proof}
Let $Z_t : = \scal{\mu_{\tz}}{X_{a_t}}$, $Z_t^* : = \scal{\mu_{\tz}}{X_{a_t^*}}$  and $\cdf_{Z_{t}}$, $\cdf_{Z_t^*}$ be their CDF conditioned on $\mu_{\tz}$, $a_t$, and $a_t^*$. Let also 
$
    \instregret(t) :=  \cdf_{a_t^*}(\scal{\mu^*}{X_{t,a_t^*}}) -\cdf_{a_t}(\scal{\mu^*}{X_{t, a_t}})\enspace.
$
Then we can write %
\begin{align*}
    \instregret(t) &= \underbrace{\cdf_{a_t^*}(\scal{\mu^*}{X_{t,a_t^*}}) -\cdf_{a_t^*}(\scal{\mu_{\tz}}{X_{t,a_t^*}})}_{\text{(I)}}+ \underbrace{\cdf_{a_t^*}(\scal{\mu_{\tz}}{X_{t,a_t^*}})-\cdf_{Z_t^*}(\scal{\mu_{\tz}}{X_{t,a_t^*}})}_{\text{(II)}}\\
    &\quad+ \underbrace{\cdf_{Z_t^*}(\scal{\mu_{\tz}}{X_{t,a_t^*}})-\hcdf_{t,a_t^*}(\scal{\mu_{\tz}}{X_{t,a_t^*}})}_{\text{(III)}}+\underbrace{\hcdf_{t,a_t^*}(\scal{\mu_{\tz}}{X_{t,a_t^*}}) - \hcdf_{t,a_t}(\scal{\mu_{\tz}}{X_{t,a_t}})}_{\text{(IV)}} \\
    &\quad+\underbrace{\hcdf_{t,a_t}(\scal{\mu_{\tz}}{X_{t,a_t}})-\cdf_{Z_t}(\scal{\mu_{\tz}}{X_{t,{a_t}}})}_{\text{(V)}}+\underbrace{\cdf_{Z_t}(\scal{\mu_{\tz}}{X_{t,{a_t}}})-\cdf_{a_t}(\scal{\mu_{\tz}}{X_{t,{a_t}}})}_{\text{(VI)}}\\
    & \quad+ \underbrace{\cdf_{a_t}(\scal{\mu_{\tz}}{X_{t,{a_t}}})-\cdf_{a_t}(\scal{\mu^*}{X_{t, a_t}})}_{\text{(VII)}}\enspace.
\end{align*}
Since $a_t$ is chosen greedily in  \Cref{alg:fairgreedy} we have $\text{(IV)}\leq 0$.
Then, applying \Cref{lm:cdfbound}\ref{lm:cdfbound_1}, Cauchy-Schwarz and $\norm{X_{t,a}}_* \leq \norm{x_{\max}}_*$ for (I) and (VII) and \Cref{lm:cdfbound}\ref{lm:cdfbound_2} for (II) and (VI), we obtain
\begin{align*}
    \text{(I) + (VII)} \leq 2M\norm{\mu^* - \mu_{\tz}}\norm{x_{\max}}_{*}\enspace, 
    \quad \text{(II) + (VI)} \leq 4 M\norm{\mu^* - \mu_{\tz}}\norm{x_{\max}}_{*}\enspace.
\end{align*}
By noticing that  $\hcdf_{t,a}(\cdot)$ is the empirical CDF of the random variable $\scal{\mu_{\tz}}{X_a}$ conditioned to $\mu_{\tz}$, we can bound (III) and (V) directly using the DKWM inequality (see \Cref{lm:dwk}), which gives that with probability at least $1-\delta/4$ and for all $t$ such that  $3 \leq t \leq T$ we have 
\begin{align*}
    \text{(III) + (V)}  \leq 2\sqrt{\frac{\log(8KT/\delta)}{t-1}}\enspace.
\end{align*}
We conclude the proof by combining the previous bounds and setting $\norm{\cdot} = \norm{\cdot}_{V_{\tz}}$ .
\end{proof}

We proceed by controlling the term $\norm{\mu^* - \mu_{\tz}}_{V_{\tz}}\norm{x_{\max}}_{V_{\tz}^{-1}}$ in \Cref{lm:instantbound}. The quantity $\norm{\mu^* - \mu_{\tz}}_{V_{\tz}}$ can be bounded using the OFUL confidence bounds \cite[Theorem 2]{abbasi2011improved}, since the noise term in $\mu_{\tz}$ decreases at an appropriate rate. Controlling $\norm{x_{\max}}_{V_{\tz}^{-1}}$ requires instead different results than the ones in \cite{abbasi2011improved}, since it depends on the distributions of $\{ X_a \}_{a=1}^K$ and not only on previous contexts and rewards. Hence, to  provide an upper bound for $\norm{x_{\max}}_{V_{\tz}^{-1}}$ which decreases with $t$, we also rely on \Cref{assump}\ref{ass:indep} and the structure of \cref{alg:fairgreedy}, which enable the following history-agnostic lower bound on the probability of selecting one arm. 
\begin{proposition}\label{prop:patlowerbound}
Let \Cref{assump} hold, $a_t$ be generated by \Cref{alg:fairgreedy} and $c \in [0,1)$. Then with probability at least $1-\delta/4\,$, for all $a \in [K]$ and all $t \geq 3 +  8\log^{3/2}\big(5K\e/\delta\big)/\big(1- \sqrt[K]{c}\big)^3$ we have
\begin{align*}
  \Pr(a_t = a \,|\, \Hminus_{t-1} ) &\geq \frac{c}{K} \,\enspace,
\end{align*}
where we recall that $\Hminus_t = \cup_{i=1}^{t} \left\{\{X_{i, a}\}_{a=1}^K, r_{i, a_{i}}, a_{i}\right\}$. 
\end{proposition}
\begin{proof}[Proof sketch (proof in \Cref{se:proofpatlowerbound})]
For any $a \in [K]$, let $\hat{r}_{t,a} = \scal{\mu_{\tz}}{X_{t,a}}$ be the estimated reward for arm $a$ at round $t$, denote with $\cdf_{\hat{r}_{t,a}}$  the CDF of $\hat{r}_{t,a}$ conditioned on $\mu_{\tz}$, and let
\begin{align*}
    \phi_{t,a} := \cdf_{\hat{r}_{t,a}}(\hat{r}_{t,a} )\enspace, \quad\text{and}\quad \hat \phi_{t,a} := \hat \cdf_{t,a}(\hat{r}_{t,a} )\enspace,
\end{align*}
where $\hat \cdf_{t,a}(\hat{r}_{t,a})$ is defined in \cref{step:ecdf} of \Cref{alg:fairgreedy}. Now, by the definition of $a_t$ (\cref{step:at} of \Cref{alg:fairgreedy}), we have
\begin{align*}
    \Pr(a_t = a \,|\, \Hminus_{t-1} ) &= \sum_{m=1}^K \frac{1}{m} \Pr(a \in C_t, |C_t| = m \,|\, \Hminus_{t-1} )\enspace,
\end{align*}
where we introduced $C_t : = \argmax_{a\in [K]} \hat\phi_{t,a}$. Let $\epsilon_t > 0$ and continue the analysis conditioning on the events where $\sup_{a\in[K]}|\phi_{t,a} - \hat \phi_{t,a}| \leq \epsilon_{t}$. Then, we can write
\begin{equation*}
   \Pr(a_t = a \,|\, \Hminus_{t-1} ) 
   \geq  \Pr(\hat \phi_{t,a} > \hat\phi_{t,a'} \,,\,\forall 
   a'\neq a\,|\, \Hminus_{t-1} ) \geq \Pr(\phi_{t,a}  > \phi_{t,a'} + 2 \epsilon_{t}\,,\,
  \forall a'\neq a\,|\Hminus_{t-1})\enspace,
\end{equation*}
where in the first inequality we considered the case when $a \in C_t$ and $|C_t| = 1$, and in the second inequality we considered the worst case scenario where $\hat \phi_{t,a} = \phi_{t,a} -\epsilon_{t}$ and $\hat \phi_{t,a'} = \phi_{t,a'} + \epsilon_{t}$. \Cref{assump}\ref{ass:abscont} and the additive noise in $\mu_{\tz}$ imply that $\scal{\mu_{\tz}}{X_a}$ is an absolutely continuous random variable for each $a \in [K]$, which yields that $\{\phi_{t,a}\}_{a\in [K]}$ is uniformly distributed on $[0,1]$. Furthermore, $\{\phi_{t,a}\}_{a\in [K]}$ are also independent due to \Cref{assump}\ref{ass:indep}. Thus we have
\begin{align*}
  \Pr(a_t = a \,|\, \Hminus_{t-1} )&\geq \int_{0}^{1}\left(\Pr(\phi_{t,a'} < \mu - 2 \epsilon_{t} )\right)^{K-1}\d\mu =\int_{2\epsilon_t}^{1}(\mu-2\epsilon_{t})^{K-1}\d\mu =  \frac{(1-2\epsilon_{t})^K}{K}\enspace.
\end{align*}
Finally, thanks to \Cref{assump}\ref{ass:iid} we can invoke the DKWM inequality to appropriately bound $\epsilon_t$ in high probability for all $t$ sufficiently large.
\end{proof}

The property in \Cref{prop:patlowerbound} guarantees that, for sufficiently large $t$, the policy can get arbitrarily close to satisfy history-agnostic demographic parity in \eqref{eq:iagm}.
In particular this allows us to control $\norm{x_{\max}}_{V_{\tz}^{-1}}$ by using a standard matrix concentration inequality \citep[Theorem 3.1]{mart_random} on a special decomposition of $V_{\tz}$, thereby enabling the following result (proof in \Cref{se:proofmuxbound}).

\begin{lemma}\label{lm:muxbound}
Let \Cref{assump} hold, $a_t$ be generated by \Cref{alg:fairgreedy}, $\tau_1 = 32K^3\log^{3/2}\big(5K\e/\delta\big)$, $\tau_2 = \frac{54\xmax^2}{\eigminplus(\Sigma)}\log(\frac{4d}{\delta})$ and $\tau = 4\max(\tau_1, \tau_2) +3$. Then, with probability at least $1-\frac{3\delta}{4}$, for all $t \geq \tau$ we have
\begin{align*}
    \norm{\mu^* - \mu_{\tz}}_{V_{\tz}}\norm{x}_{V_{\tz}^{-1}} \leq \frac{8\xmax}{\sqrt{\eigminplus(\Sigma)\cdot t}}\left[b_1\sqrt{d\log((8+4t\max(\xmax^2/\lambda,1))/\delta)} +\lambda^{\frac{1}{2}}\norm{\mu^*}_2\right]\enspace,
\end{align*}
where $b_1 = \lambda^{\frac{1}{2}}+R + \xmax$, $\Sigma := K^{-1}\sum_{a=1}^{K} \E\big[X_a X^\top_a\big]$ and $\eigminplus(\Sigma)$ is its smallest nonzero eigenvalue.
\end{lemma}

Finally we obtain the desired high probability regret bound by combining  \Cref{lm:instantbound} with \Cref{lm:muxbound} and summing over the $T$ rounds (see \Cref{se:proofregret} for a proof).

\begin{theorem}\label{th:regret}
Let \Cref{assump} hold and $a_t$ be generated by \Cref{alg:fairgreedy}. Then, with probability at least $1-\delta$, for any $T\geq 1$ we have
\begin{align*}
     \reg(T) \leq&\frac{96M\xmax}{\sqrt{\eigminplus(\Sigma)}}\left[(\lambda^{\frac{1}{2}}+R + \xmax)\sqrt{dT\log((8+4T\max(\xmax^2/\lambda,1))/\delta)} +\sqrt{\lambda T}\norm{\mu^*}_2\right] \\
     & \quad+8\sqrt{\frac{T\log(8KT/\delta)}{3}} + \tau\enspace,
\end{align*}
with $\tau$ defined in \Cref{lm:muxbound}. Hence $\reg(T) = O ( K^3\log^{3/2}(K/\delta) + \sqrt{dT\log(KT/\delta)})$.
\end{theorem}
The regret bound in \Cref{th:regret} has two terms. The $O(K^3\log^{3/2}(K))$ term describes the rounds needed to satisfy \Cref{prop:patlowerbound} with $c=1/2$. The remaining part, which is of order $O(\sqrt{dT\log(KT)})$ is instead associated to the convergence of the empirical CDF and  to the bandit performance. Indeed, it recalls the standard regret bound holding for finite-action linear contextual bandits \cite{auer2002using,chu2011contextual,lattimore2020bandit}. 

\vspace{-0.1truecm}
\section{Simulation with Diverse Reward Distributions}\label{se:simu}
\vspace{-0.1truecm}

\begin{figure}[t]
     \centering
     \begin{subfigure}[b]{0.34\textwidth}
         \centering
\includegraphics[width=\textwidth]{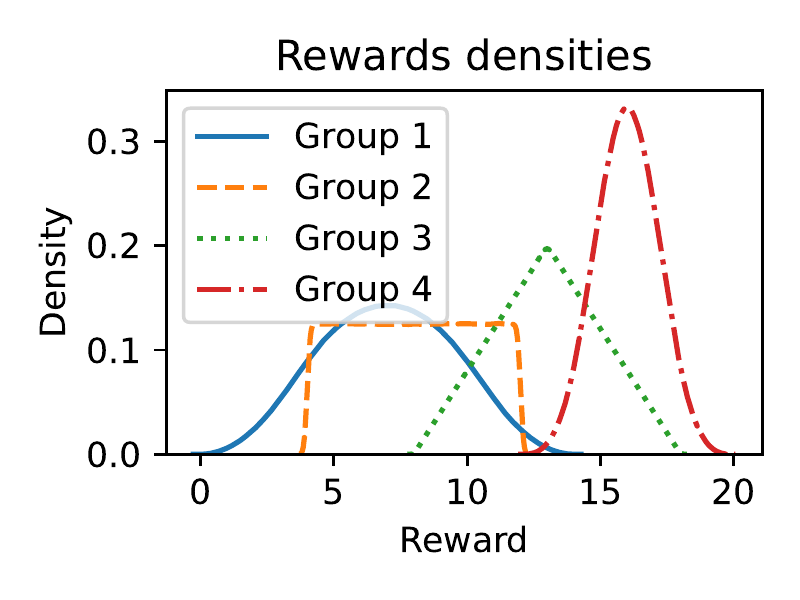}
     \end{subfigure}
     \hfill
     \hspace{-.5cm}
     \begin{subfigure}[b]{0.34\textwidth}
         \centering
         \includegraphics[width=\textwidth]{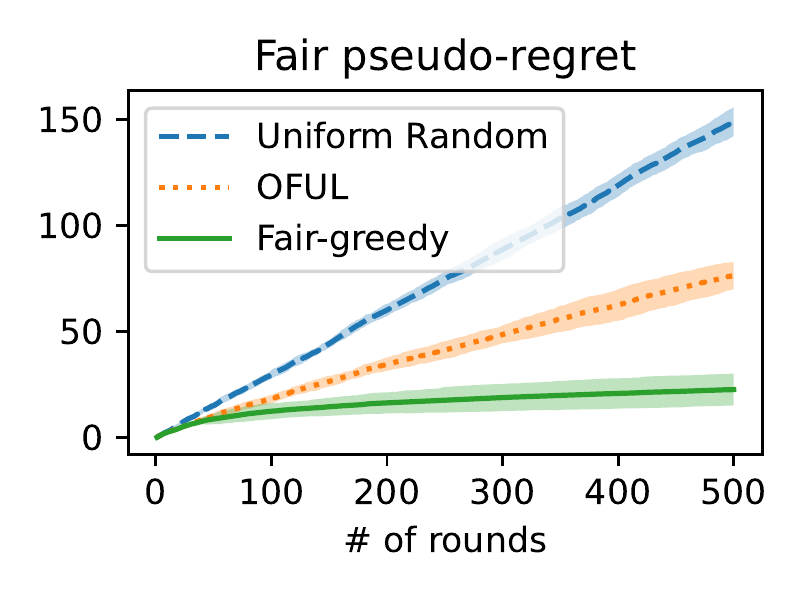}
     \end{subfigure}
     \hfill
     \hspace{-.5cm}
     \begin{subfigure}[b]{0.34\textwidth}
         \centering
         \includegraphics[width=\textwidth]{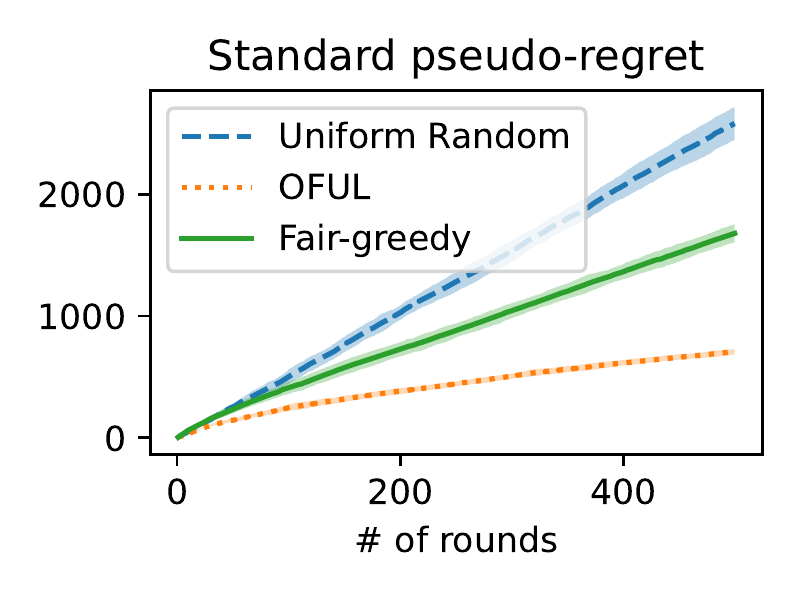}
     \end{subfigure}
        \caption{
        \small \textbf{Simulation Results}. First image is a density plot of the reward distributions while the second and third plot show the standard and fair pseudo-regrets, with mean (solid lines) $\pm$ standard deviation (shaded region) over 10 runs. To approximate the true reward CDF for each group we use the empirical CDF with $10^7$ samples.
        \vspace{-.3cm}
        }
        \label{fig:all}
\end{figure}

We present an illustrative proof of concept experiment which simulates groups with diverse reward distributions. We set $K=4$,  $\eta_t = 2 \xi_t$, where $\xi_t$ has standard normal distribution, $X_a = B_a Y_a + c_a$ where  each coordinate of $Y_a \in \R^4$ is an independent sample from the uniform distribution on $[0,1]$, $B_a \in R^{(4K+1) \times 4}$ is such that $X_a$ contains $Y_a$ starting from the $4a$-th coordinate and $c_{a}$ has all the coordinates set to zero except for the last which is set to $3a$ to simulate a group bias.  In this setup $\mu^*$ acts differently on each group, in particular, we note that $\mu^* \in R^{4K+1}$ has its last coordinate multiplying the group bias in $c_a$, which we set to $1$, and $4$ group-specific coordinates, which we set to  manually picked values between $0$ and $9$. Results are shown in \Cref{fig:all}, where we compare our greedy policy in \Cref{alg:fairgreedy} with OFUL \cite{abbasi2011improved}, both with regularization parameter set to $0.1$, and with the Uniform Random policy. 
We observe that, as expected from our analysis, our policy achieves sublinear fair pseudo-regret, while also having better-than random, although linear, standard regret. Additional details and an experiment on US census data with gender as the sensitive group are in \Cref{se:expadditional}.

\vspace{-0.1truecm}
\section{Multiple Candidates for Each Group}\label{se:multiplemain}
\vspace{-0.1truecm}

In this section, we analyze the more realistic case where contexts from a given arm do not necessarily belong to the same group. The complete analysis is presented in \Cref{se:multtiple}.
 In particular, we assume that at each round $t$, the agent receives $\{(X_{t,a}, s_{t,a})\}_{a=1}^K$, which are $K$ i.i.d. random variables where $s_{t,a} \in [G]$ is the sensitive group of the context $X_{t,a} \in \R^d$ and $G$ is the total number of groups.
This setting can model for example a hiring scenario where at each round the employer has to choose among candidates belonging to different ethnic groups, some of which are minorities and hence have a small probability $\Pr(s_{t,a} = i)$ of being in the pool of received candidates. By naturally adapting the definition of fair-regret $\reg(T)$, the Fair-Greedy policy and \Cref{assump} to this setting, with probability $1-\delta$  we obtain the following regret bound   (see \Cref{corr:equalprob} in \Cref{se:multtiple}).

\begin{equation}\label{eq:boundmultiple}
    \reg(T) = O \Bigg( \frac{G\log(GT/\delta)}{K q_{\min}} + \frac{(KG)^{3/2}\log^{3/2}(G/\delta)}{q_{\min}^{3/2}} + \sqrt{\frac{dT\log\left(GT/\delta\right)}{(1 + K/G)q_{\min})}}  \Bigg)
\end{equation}

\begin{figure}[t!]
     \centering
     \begin{subfigure}[b]{0.405\textwidth}
         \centering
         \includegraphics[width=\textwidth]{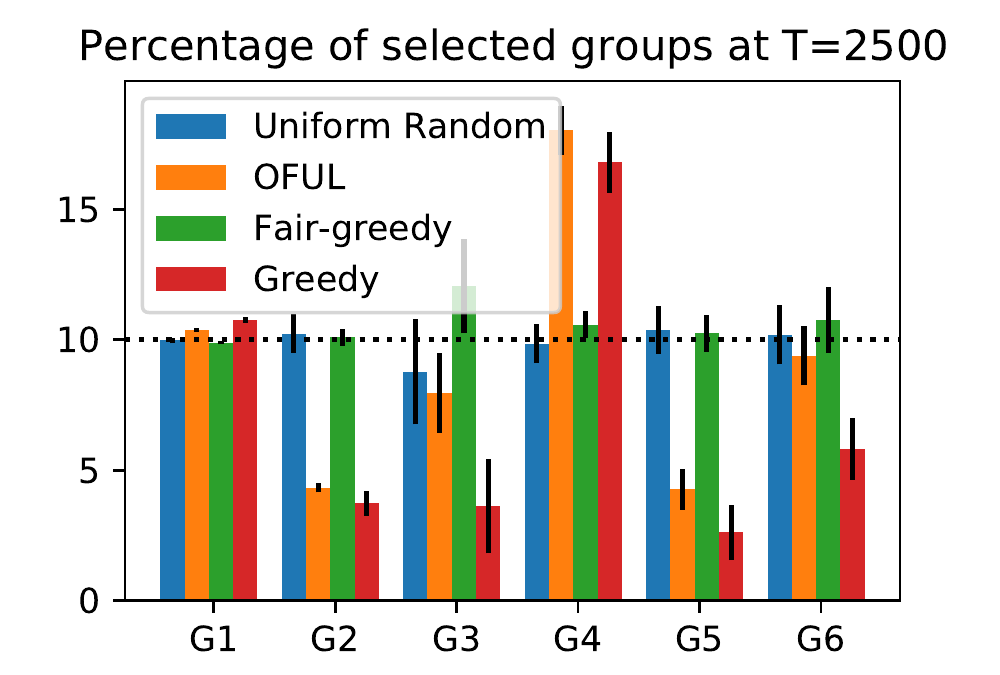}
     \end{subfigure}
    \hspace{-.35cm}
    \begin{subfigure}[b]{0.305\textwidth}
    \centering
    \includegraphics[width=\textwidth]
    {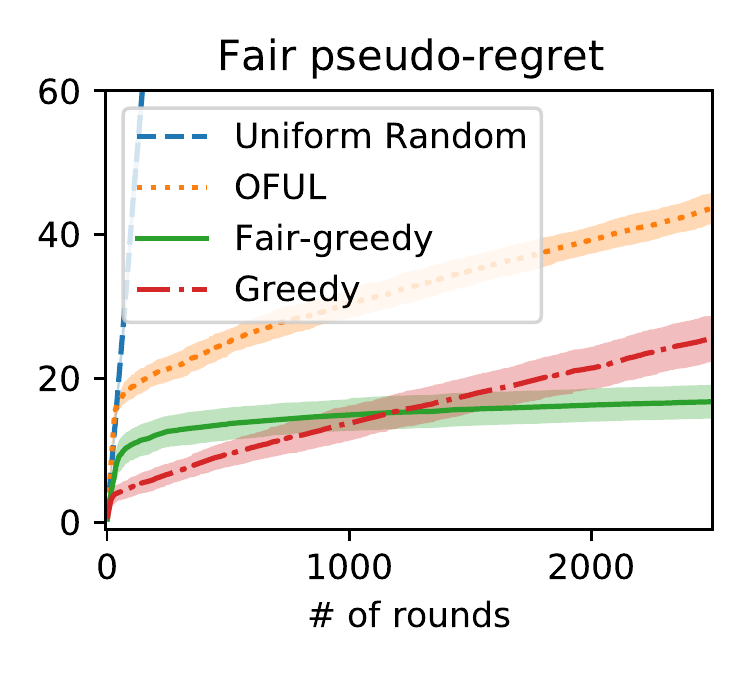}
     \end{subfigure}
    \hspace{-.35cm}
     \begin{subfigure}[b]{0.305\textwidth}
         \centering
         \includegraphics[width=\textwidth]
         {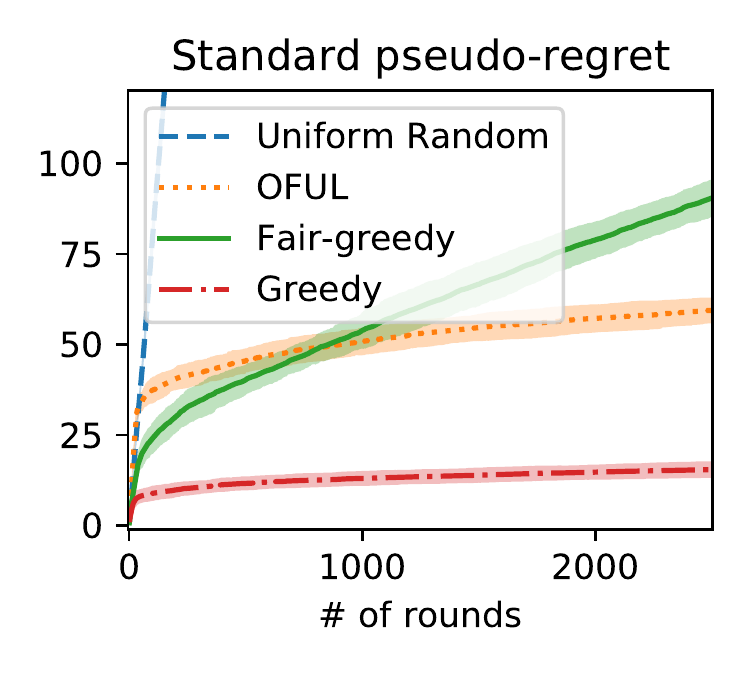}
     \end{subfigure}
        \caption{\small \textbf{US Census Results. Group $=$ Ethnicity}. 
        First image shows mean (colored bars) and std (thinner black bars), while the other two show the mean (solid lines) $\pm$ standard deviation (shaded region) over 10 runs. To compute the reward CDF for each group we use the empirical CDF on $5K$ samples from $D2$. Percentage of selected groups is computed by dividing the number of candidates of a given group selected by the policy by the total number of candidates of that group received by the agent. G$X$ with $X \in \{1, \dots, 6\}$, stands for group $X$.
        }
        \label{fig:expcensusmulti}
 \end{figure}

 where $q_{\min} = \min_{i \in [G]} \Pr(s_{t,a} = i) G$, so that $q_{\min} = 1$ if and only if each group has equal probability of being sampled and $q_{\min} > 0$  without loss of generality. \eqref{eq:boundmultiple} is similar to \Cref{th:regret}, having the same dependency on $\delta$ and $T$ but an improved dependency on the number of arms $K$ when  $K > G$, since contexts from all arms can be used to estimate the CDF of each group. The first term in \eqref{eq:boundmultiple} comes from the application of the Chernoff bound to lower bound the number of candidates in each group received by the agent, which is now random.

\textbf{US Census experiments. Group = Ethnicity.} We test this setting in practice by simulating the hiring scenario discussed above with data from the US Census containing the income and other useful indicators of several individuals in the United States. This data is accessed via the FolkTables library \citep{ding2021retiring}. In particular, at each round, we sample $K=10$ candidates at random from the population containing the $G=6$ largest ethnic groups\footnote{We remove groups with less than $5$K individuals to compute accurately the true CDFs for the fair regret.}, the reward is a previously computed linear estimate of the income, while the noisy reward is the true reward plus some small gaussian noise. 
We compare the Fair-Greedy Policy with OFUL \citep{abbasi2011improved}, Greedy (selects the candidate with the best estimated reward) and Uniform Random in \Cref{fig:expcensusmulti}. Similarly to the synthetic experiment in \Cref{se:simu}, the Fair-greedy policy achieves the best fair pseudo-regret and standard regret better than Uniform Random. Note that Greedy outperforms OFUL, which is too conservative in this scenario. Furthermore, the Fair-Greedy policy selects approximately the same percentage of candidates from each group, similarly to Uniform Random, while OFUL and Greedy select smaller percentages from G2, G3, G5 and G6. In \Cref{se:expmultigroup} we provide more details and a comparison with two oracle fair policies which shows that knowing $\mu^*$ plays a more important role than knowing the true reward CDFs of each group.

\section{Conclusions and Future Work}\label{se:conclusions}
We introduced the concept of group meritocratic fairness in linear contextual bandits, which states that a fair policy should select, at each round, the candidate with the highest relative rank in the pool. This allows us to compare candidates coming from different  sensitive groups, but it is hard to satisfy since the relative rank is not directly observed and depends on both the underlying reward model and on the rewards distribution for each group. After defining an appropriate fair pseudo-regret we analyzed a greedy policy and proved that  its fair pseudo-regret is sublinear with high probability. 

This result was possible since we restricted the analysis to the case where the contexts of different groups are independent random variables and the rewards are absolutely continuous. Relaxing these assumptions is a challenging avenue for future work. 
In particular, without the independence of contexts across arms, different approaches relying on confidence intervals might be necessary. Other two interesting directions are (i) to study the optimality of the proposed results and establishing lower bounds for any algorithm which minimises the fair pseudo-regret and (ii) to design a learning policy which aims at achieving a tradeoff between group meritocratic fairness and reward maximization.

\RG{Are lower bounds not trivial in this setting? are they trivial in T? Not clear to me.}

\textbf{Acknowledgments.} This work was supported in part by the EU Projects ELISE and ELSA.

\newpage

\bibliographystyle{plain}
\bibliography{bibliography}

\iftrue

\section*{Checklist}

\begin{enumerate}

\item For all authors...
\begin{enumerate}
  \item Do the main claims made in the abstract and introduction accurately reflect the paper's contributions and scope?
    \answerYes{}
  \item Did you describe the limitations of your work?
    \answerYes{}
  \item Did you discuss any potential negative societal impacts of your work?
    \answerNA{}
  \item Have you read the ethics review guidelines and ensured that your paper conforms to them?
    \answerYes{}
\end{enumerate}

\item If you are including theoretical results...
\begin{enumerate}
  \item Did you state the full set of assumptions of all theoretical results?
    \answerYes{}
        \item Did you include complete proofs of all theoretical results?
    \answerYes{}
\end{enumerate}

\item If you ran experiments...
\begin{enumerate}
  \item Did you include the code, data, and instructions needed to reproduce the main experimental results (either in the supplemental material or as a URL)?
    \answerYes{} at \url{https://github.com/CSML-IIT-UCL/GMFbandits}
  \item Did you specify all the training details (e.g., data splits, hyperparameters, how they were chosen)?
    \answerYes{}
        \item Did you report error bars (e.g., with respect to the random seed after running experiments multiple times)?
    \answerYes{}
        \item Did you include the total amount of compute and the type of resources used (e.g., type of GPUs, internal cluster, or cloud provider)?
    \answerNo{}
\end{enumerate}

\item If you are using existing assets (e.g., code, data, models) or curating/releasing new assets...
\begin{enumerate}
  \item If your work uses existing assets, did you cite the creators?
    \answerYes{}
  \item Did you mention the license of the assets?
    \answerNo{}
  \item Did you include any new assets either in the supplemental material or as a URL?
    \answerNA{}
  \item Did you discuss whether and how consent was obtained from people whose data you're using/curating?
    \answerNA{}
  \item Did you discuss whether the data you are using/curating contains personally identifiable information or offensive content?
    \answerNA{}
\end{enumerate}

\item If you used crowdsourcing or conducted research with human subjects...
\begin{enumerate}
  \item Did you include the full text of instructions given to participants and screenshots, if applicable?
    \answerNA{}
  \item Did you describe any potential participant risks, with links to Institutional Review Board (IRB) approvals, if applicable?
    \answerNA{}
  \item Did you include the estimated hourly wage paid to participants and the total amount spent on participant compensation?
    \answerNA{}
\end{enumerate}
\end{enumerate}

\fi

\appendix

\newpage
\appendix
\begin{center}
    {\Large \bf Appendices}
\end{center}
\vspace{.3truecm}

Below we give an overview of the structure of the Appendices.

\begin{itemize}
\item In  \Cref{se:auxlemmas} we present some auxiliary lemmas which are used to prove our results.
      
\item In \cref{se:proofDGP} we present the proof of \Cref{DGP}. 

\item \Cref{se:proofsregret} presents the proofs for the results in \Cref{sec:regret}.

\item In \Cref{se:expadditional} we include additional details on the experimental setting in \Cref{se:simu} and an experiment on the US Census Data with gender as sensitive attribute.

\item In \Cref{se:multtiple} we provide a more rigorous treatment of the setting introduced in  \Cref{se:multiplemain}, where we can have multiple condidates for each group at any given round. We also provide more details and additional comparisons for the experiment on the US Census Data in \Cref{se:multiplemain}.

\item In \Cref{sec:linreg} we investigate the tradeoff between fair an standard regret.

\end{itemize}

\section{Auxiliary Lemmas}\label{se:auxlemmas}

\begin{lemma}\label{lm:prodabs}
Let $n \in \N$, and assume that $Y, \nu$ are independent random variables in $\mathbb{R}^n$, such that $Y$ is absolutely continuous and $\nu \neq 0$, almost surely. Then, $\nu^{\top}Y$ is an absolutely continuous random variable. 
\end{lemma}
\begin{proof}
It is enough to show that for any $A\subset \mathbb{R}$ with zero Lebesgue measure, $\Pr(\nu^{\top}Y \in A) = 0$. Let $A\subseteq \mathbb{R}$, then we can write 
\begin{align*}
    \Pr(\nu^{\top}Y \in A) = \E[\Pr(\nu^{\top}Y \in A \,|\, \nu)]\enspace.
\end{align*}
We proceed the proof by controlling the term $\Pr(\nu^{\top}Y \in A \,|\, \nu)$. We know that $\nu \neq$ 0 almost surely. Now, since $Y$ and $\nu$ are independent, let $\nu=w$ for a fixed $w \in \mathbb{R}^n$  such that $w\neq 0$, then we have that
\begin{align*}
    \Pr(w^{\top}Y \in A) = \int_{y \in \mathbb{R}^n}\indic{\frac{w^{\top}}{\norm{w}_{2}}Y \in A'}f_{Y}(y)\d y\enspace,
\end{align*}
where we defined $A':=\left\{\frac{x}{\norm{w}_{2}}:x\in A\right\}$. Now consider the change of basis matrix $R=(v_1,\dots,v_n)^\top$, such that $v_1 = \frac{w}{\norm{w}_{2}}$, with $RR^{\top} =\mathbbm{I}_n$. By assigning $\hat{Y} =R^{\top}Y$, we can write
\begin{align*}
    \Pr(w^{\top}Y \in A) = \int_{\hat{y} \in \mathbb{R}^n}\indic{\hat{y}_1 \in A'}f_{Y}(R\hat{y})\d \hat{y}\enspace.
\end{align*}
Since we assume that $Y$ is an absolutely continuous random variable, there exists $M_Y>0$, such that $\sup_{y \in \mathbb{R}^n}f_Y(y) \leq M_Y$ almost surely, which allows us to write
\begin{align*}
    \Pr(w^{\top}Y \in A) \leq M_Y\int_{\hat{y} \in \mathbb{R}^n}\indic{\hat{y}_1 \in A'}\d \hat{y}\enspace.
\end{align*}
Finally, it is straightforward to check that if $A$ has a zero Lebesgue measure, then $A'$ also has a a zero Lebesgue measure, which gives $\Pr(w^{\top}Y \in A) = 0$.
\end{proof}

\begin{lemma}[Lipschitz CDF]\label{lm:cdflip}
Let $n \in \N$, $\nu \in \R^n/\{0\}$ and $b \in \R$. Let also $Y$ be an absolutely continuous random variable with values in $\R^n$, with probability density function $\pdf_{Y}$. Then the CDF of $ Z = \scal{\nu}{Y} + b$, namely $\cdf_Z$, is Lipschitz continuous. More specifically
\begin{equation*}
    \abs{\cdf_Z(r) - \cdf_Z(r')} \leq M'|r - r'| \qquad \forall r, r' \in \R\enspace,
\end{equation*}
where $M' = \max_{z\in \R} \pdf_{Z}(z)$.
\end{lemma}
\begin{proof}
Since $\nu \neq 0$ and $Y$ is absolutely continuous, $Z$ is also absolutely continuous with probability density $f_{Z}$ (see \Cref{lm:prodabs}).  Furthermore, if $r' \leq r$, we can write
\begin{align*}
     \cdf_{Z}(r) - \cdf_{Z}(r') 
    &= \int_{-\infty}^{r} f_{Z}(t)\d t - \int_{-\infty}^{r'} f_{Z}(t)\d t
     = \int_{r'}^{r} f_{Z}(t)\d t 
    \leq M' (r-r').
\end{align*}
Applying the same reasoning to the case when $r \leq r'$ concludes the proof.
\end{proof}

\begin{lemma}\label{suppcov}
Let $\{X_a\}_{a=1}^K$ be $K$ random variables with values in $\R^d$ and such that they are all $0$ with probability strictly less than one. Define $\Sigma = K^{-1}\sum_{a=1}^{K}\E[X_a X_a^\top]$ and let $\Sigma = U S U^\top$  be its compact eigenvalue decomposition with $U \in \R^{d\times r}$, $S \in \R^{r \times r}$ with $1 \leq r \leq d$. Assume that $S$ is invertible. Then, for any $y \in \cup_{a=1}^{K}{\rm Supp}(X_a)$, we have $U U^\top y = y$ and  $\eigminplus(\Sigma)\norm{y}_2^2 \leq y^{\top}\Sigma y$, where $\eigminplus(\Sigma)$ is the smallest non-zero eigenvalue of the matrix $\Sigma$.
\end{lemma}

\begin{proof}
Let $X$ be a random variable with the distribution $\Pr(X) = K^{-1}\sum_{a=1}^{K}\Pr(X_a)$. It is straightforward to check that $\Sigma = \E[XX^{\top}]$, and $y \in {\rm Supp}(X)$. We can also write $y = y_1 + y_2$ where $y_2 \in \Ker(\Sigma) : = \{z \in \R^d \,:\, \Sigma z = 0\}$ and $y_1 \in \Ker(\Sigma)^{\perp} := \{ z \in \R^d \,:\, \scal{z}{x} = 0, \forall x \in \Ker(\Sigma)\}$. This implies that 
\begin{equation}\label{eq:absone}
    y_2^\top \Sigma y_2 = \E[y_2^\top X X^\top y_2] = 0\enspace.
\end{equation}
Now, let $f(x) = (y_2^{\top}x)^2$. Then, $f(x) \geq 0$, for any $x \in \mathbb{R}^d$ and $f(y) = \norm{y_2}_2^4$. Furthermore, since, $f(x)$ is a continuous function there exists $\epsilon> 0$, such that for any $z \in B(y,\epsilon) = \{x \in \mathbb{R}^d: \norm{x - y}_2<\epsilon\}$, $f(z) \geq \frac{\norm{y_2}_2^4}{2}$. 
On the other hand, since $y \in {\rm Supp}(X)$, $\Pr(X \in B(y,\epsilon))>0$. Hence, we can write
\begin{align*}
    0 = y_2^{\top}\Sigma y_2 = \E[f(X)] \geq \E[f(X) \indic{X \in B(y,\epsilon)}] 
    \geq \frac{\norm{y_2}_2^4}{2}\Pr(X \in B(y,\epsilon))\enspace,
\end{align*}
therefore $y_2 = 0$ which implies that $y \in \Ker(\Sigma)^{\perp}$. 
Since $UU^\top y$ is the orthogonal projection of $y$ onto $\Ker(\Sigma)^{\perp}$ we conclude that $y= UU^\top y$, $y^{\top}= y^{\top}UU^\top$ and
\begin{equation*}
    y^\top\Sigma y = y^\top U S U^\top y \geq \eigminplus(\Sigma) y^\top U U^\top y = \eigminplus(\Sigma) \norm{y}_2^2\enspace.
\end{equation*}
\end{proof}

\section{Proof of Lemma~\ref{DGP}}\label{se:proofDGP}
\begin{proof}
If $t < 3$ then $\mu_{\tz} = 0$ and $a_t \sim \uniform[[K]]$ and the statement follows.
If $t \geq 3$, Let $\mu \in \mathcal{S} := \{ \mu' \in \R^d \,:\, \mu'^{\top} B_a \neq 0 \,\forall a \in [K] \}$, $\hat r_{i,a} = \scal{\mu}{X_{i,a}}$ and $t' = t - \tz -1$.
Then by \Cref{lm:prodabs} $\hat r_{i,a}$ is absolutely continuous. Given a permutation of indices $j = (j_1, \dots, j_{t'} )$ where $j_i \in \{\tz+1, \dots ,t\}$, for $i\in [t']$. Let $\Omega_a$ be the set of the events of $\{X_{i,a}\}_{i=\tz+1}^{t}$ and $P$ be the set of all permutations of the indices $\{\tz+1, \dots ,t\}$.  Consider the event 
\begin{align*}
    E_{a, j} = \{ \omega \in \Omega_a \, :  \hat r_{j_1,a}  <\dots < \hat r_{j_{t'},a} \, \}\enspace.
\end{align*} 
Since $\{\hat r_{i,a}\}_{i=\tz+1}^{t}$ are absolutely continuous, we have for all $k\neq i$, $\Pr(\hat{r}_{j_{i},a} = \hat{r}_{j_{k},a})=0$ and this yields  $\Omega_a = \cup_{j \in P} E_{a, j}$ and $E_{a, j} \cap E_{a, j'} = \emptyset$ for all $j \neq j'$. Furthermore, since $\{\hat r_{i,a}\}_{i=\tz+1}^{t}$ are i.i.d. we have that $p_a : = \Pr(E_{a, j}) = \Pr(E_{a, j'})$ for all $j \neq j'$. In particular, since $|P| = t'!$ we have $p_a = 1/(t'!)$.

Let $\phi_a = (t'-1)^{-1} \sum_{i=\tz+1}^{t-1} \indic{ \hat r_{i,a} < \hat r_{t,a}}$. Let $b \in \{0, \dots, t-1\}$ and let $P_b = \{j \in P \,:\,  j_{b+1} = t\}$. We have that $|P_b| = (t'-1)!$ and
\begin{align*}
\Pr(\phi_a = b/(t'-1)) = \sum_{j \in P_b}\Pr( E_{a,j}) = (t'-1)! p_a = \frac{1}{t'}\enspace.
\end{align*}
As a consequence, for all $a \in [K]$, $\phi_{a}$ is uniform over $\{0, 1/(t'-1), \dots, 1\}$. Since $\{r_{i,a}\}_{i \in [t+1], a \in [K]}$ are mutually independent we have that $\{\phi_a\}_{a \in [K]}$ are i.i.d. discrete uniform random variables. As a consequence, let $\hat a = \mathcal{U}\left[\argmax_{a'\in [K]} \hat \phi_{a} \right]$ we have that $\Pr(\hat a = a) = 1/K$. Using the definition of $\hat a$ we have
\begin{align*}
    \Pr(a_t = a) = \frac{1}{K} \Pr(\mu_{\tz} \in \mathcal{S}) + \Pr(a_t = a \,|\, \mu_{\tz} \in \mathcal{S}^c) \Pr(\mu_{\tz} \in \mathcal{S}^c) = \frac{1}{K}\enspace,
\end{align*}
where the last equality is derived by the fact that by the construction of $\mu_{\tz}$, $\Pr(\mu_{\tz} \in \mathcal{S}) = 1$.
\end{proof}

\section{Proofs of Results in Sec.~\ref{sec:regret}}\label{se:proofsregret}

The following result is used in the Proof of \Cref{lm:instantbound}. Its proof is obtained by using the Dvoretzky–Kiefer–Wolfowitz-Massart inequality \cite{Dvoretzky_Kiefer_Wolfowitz56,DKWM} combined with a union bound.
\RG{Can it be improved using stopping times to change the dependency from $\log(T)$ to $\log(t)$?}

\begin{lemma}\label{lm:dwk}
Let \Cref{assump}\ref{ass:iid} hold and $\hcdf_{t,a}(r)$, and $\mu_{\tz}$ to be generated by \Cref{alg:fairgreedy}. Let $Z_a : = \scal{\mu_{\tz}}{X_a}$ and denote with $\cdf_{Z_a}$ its CDF, conditioned on $\mu_{\tz}$. Then, with probability at least $1-\delta$ we have that for all $3 \leq t \leq T$
 \begin{equation*}
     \sup_{a\in[K], r \in \R} \abs{\hcdf_{t,a}(r)-\cdf_{Z_a}(r)} \leq \sqrt{\frac{\log(2KT/\delta)}{t-1}}\enspace.
 \end{equation*}
\end{lemma}
\begin{proof}
Let $3 \leq t \in T$ and recall that $\tz = \floor{(t-1)/2}$. Note that from \Cref{assump}\ref{ass:iid}, for all $a \in [K]$,  $\{\scal{\mu_{\tz}}{X_{i,a}}\}_{i=\tz+1}^{t-1}$ are i.i.d copies of $Z_a$, conditioned on $\mu_{\tz}$. Since $\hcdf_{t,a}(r)$ is the empirical CDF of  $Z_a$ conditioned on $\mu_{\tz}$, we can apply the Dvoretzky–Kiefer–Wolfowitz-Massart inequality \cite{Dvoretzky_Kiefer_Wolfowitz56,DKWM}  to obtain
\begin{equation*}
    \Pr\left(\sup_{r \in \R} \abs{\hcdf_{t,a}(r)-\cdf_{Z_a}(r)} \geq \sqrt{\frac{\log(2/\delta')}{2(t-1-\tz)} }  \right) \leq \delta'\enspace.
\end{equation*}
Therefore, since $\tz \leq (t-1)/2$ we deduce that 
$
    \Pr\left[ E_{t,a}\right] 
    \leq \delta'\enspace,
$
where  
\begin{equation*}
E_{t,a} = \left\{ \{X_{i,a}\}_{i=\tz+1}^{t-1}, \mu_{\tz}  \,:\,\sup_{r \in \R} \abs{\hcdf_{t,a}(r)-\cdf_{Z_a}(r)} \geq \sqrt{\frac{\log(2/\delta)}{t-1}} \right\}\enspace.
\end{equation*}
Consequently, by applying a union bound we obtain
\begin{align*}
    \Pr\left[\cup_{i=1}^{T}\cup_{a=1}^{K}E_{t,a}\right] \leq \sum_{i=1}^{T}\sum_{a=1}^{K} \Pr(E_{t,a}) \leq KT\delta'\enspace,
\end{align*}
Finally, by substituting $\delta' = \delta/(KT)$ and computing the probability of the complement of $\cup_{i=1}^{T}\cup_{a=1}^{K}E_{t,a}$, we obtain the desired result.
\end{proof}

\subsection{Proof of Lemma~\ref{lm:cdfbound}}\label{se:proodcdfbound}

\begin{proof}
For every $a \in [K]$, by \Cref{assump}\ref{ass:abscont}, we have that $X_a = B_{a}Y_a  + c_a$ where $Y_a \in \R^{d_a}$ is absolutely continuous with density $f_a$. Let $\nu^* := \mu^{*\top}B_a$, $\nu := \mu^{\top}B_a$, $b^* := \scal{\mu^*}{c_a}$, $b := \scal{\mu}{c_a}$. Then we have
\begin{align*}
   Z_a =  \scal{\mu^*}{X_a} = \scal{\nu^*}{ Y_a} + b^*\,,\quad\text{and}\quad \tilde{Z}_a= \scal{\mu}{X_a} = \scal{\nu}{ Y_a} + b\enspace.
\end{align*}
From \Cref{assump}\ref{ass:abscont} we also have that $\nu^* \neq 0$, hence, by applying \Cref{lm:cdflip} with $\nu=\nu^*$ and $Y = Y_a$ and by taking the maximum over $a \in [K]$, the statement \ref{lm:cdfbound_1} follows. 

We now prove \ref{lm:cdfbound_2}. Since $Y_a$ is absolutely continuous we can write for any $r \in \R$
\begin{align*}
    \abs{\cdf_a(r) - \cdf_{\tilde{Z}_a}(r)}
    &= \abs{\cdf_{\scal{\nu^*}{Y_a} + b^*}(r) - \cdf_{\scal{\nu}{Y_a} + b}(r)} \\ 
    &\leq {\int_{y \in \R^{d_a}}} \abs[\bigg]{\indic{\scal{\nu^*}{y} + b^* \leq r} - \indic{\scal{\nu}{y} + b \leq r}} \pdf_a(y)\d y\enspace.  
\end{align*}
Now, by adding and subtracting $q(y) := \scal{\nu^* - \nu}{y} + b^* - b$ and letting $r' := r-b^*$ we have
\begin{align*}
    \abs{\cdf_a(r) - \cdf_{\tilde{Z}_a}(r)}
   &\leq {\int_{y \in \R^{d_a}}} \abs[\bigg]{\indic{\scal{\nu^*}{y} \leq r'} - \indic{\scal{\nu^*}{y} \leq r' + q(y)} } \pdf_a(y)\d y
   \\&\leq{\int_{y \in \R^{d_a}}}\indic{ r' - \abs{q(y)}\leq \scal{\nu^*}{y} \leq r' + \abs{q(y)}}\pdf_a(y)\d y\enspace.
\end{align*}
By Cauchy-Schwarz inequality, for any $y \in {\rm Supp}(Y_a)$, we get
\begin{align*}
    \abs{q(y)} \leq \abs{\scal{\nu^* - \nu}{y} + b^* - b} \leq \abs{\scal{\mu^* - \mu}{B_a y + c_a}}
    \leq \norm{\mu^* - \mu}\norm{x_{\max}}_{*} \enspace,
\end{align*}
where we defined $\norm{x_{\max}}_{*} := \max_{x \in \cup_{a = 1}^{K}\text{Supp}(X_a)}\norm{x}_* = \max_{y \in \cup_{a = 1}^{K}\text{Supp}(Y_a)}\norm{B_a y + c_a}_*$. 
We now let  $\kappa = \norm{\mu^* - \mu} \norm{x_{\max}}_{*}$, and note that
\begin{align*}
    \abs{\cdf_a(r) - \cdf_{\tilde{Z}_a}(r)} &\leq {\int_{y \in \R^{d_a}}}\indic{ r' - \kappa\leq \scal{\nu^*}{y} \leq r' + \kappa}\pdf_a(y)\d y\enspace.
\end{align*}
To control the above integral, we provide a proper change of variables. To this end, since by the assumption $\nu^* \neq 0$, we let $\{v_1,\dots,v_{d_a}\}$ be an orthonormal basis of $\R^{d_a}$, with  $v_1= \nu^*/\norm{\nu^*}_2$. Moreover, let $R = (v_1, \dots, v_{d_a})$ to be
the corresponding change of basis matrix. Then, for all $y \in \R^{d_a}$, we can always write $y = R \hat y$, where $\hat y_i = \scal{y}{v_i}$, with $\hat y_1 = \frac{\scal{\nu^*}{y}}{\norm{\nu^*}_2}$. Hence we denote with $\hat Y_a = R^\top Y_a$ which now has the first coordinate parallel to $\nu^*$.
Using the change of variables formula for multivariate integrals and noting that we are applying a rotation and hence $\abs{\text{det}(R)} = 1$ and $f_{R \hat Y_a}(R \hat y) = f_{\hat Y_a}(\hat y)/ \abs{\text{det}(R)} = f_{\hat Y_a}(\hat y)$  we get
\begin{align*}
    \abs{\cdf_a(r) - \cdf_{\tilde{Z}_a}(r)} &\leq
    \int_{\hat y\in \R^{d_a}} \indic{ \frac{r' - \kappa}{\norm{\nu^*}_2}\leq \hat y_1 \leq \frac{r' + \kappa}{{\norm{\nu^*}_2}}} \pdf_{\hat Y_a}(\hat y) \d 
    \hat y\enspace.
\end{align*}
Let $z = (\hat y_2, \dots, \hat y_{d_a})$. By Fubini's Theorem, and with the convention that $f_{\hat Y_a}(\hat y_1, z) = f_{\hat Y}(\hat y_1, z_1, \dots, z_{d_a-1})$ we have
\begin{align*}
    \abs{\cdf_a(r) - \cdf_{\tilde{Z}_a}(r)} 
    &\leq \int_{\hat y_1\in \R}\indic{ \frac{r' - \kappa}{\norm{\nu^*}_2}\leq \hat y_1 \leq \frac{r' + \kappa}{{\norm{\nu^*}_2}}} \int_{z \in \R^{d_a-1}} f_{\hat Y_a} (\hat y_1, z) \d z \d \hat y_1
    \\&= \int_{\hat y_1\in \R} \indic{ \frac{r' - \kappa}{\norm{\nu^*}_2}\leq \hat y_1 \leq \frac{r' + \kappa}{{\norm{\nu^*}_2}}}f_{\hat Y_1}(\hat y_1)\d \hat y_1
    \enspace,
\end{align*} 
where $f_{\hat Y_1}(\hat y_1) : =\int_{z \in \R^{d_a-1}} f_{\hat Y} (\hat y_1, z_1, \dots, z_{d_a-1}) \d z $ is the marginal density of $\hat Y_1 = \frac{\scal{\nu^*}{Y_a}}{\norm{\nu^*}_2}$,  and we highlight that $\hat Y_1 = (Z_a - b^*)/\norm{\nu^*}_2$. Finally note that 
\begin{align*}
 \max_{y_1 \in \R} f_{\hat{Y}_1}(y_1) = \norm{\nu^*}_2 \max_{y_1 \in \R}  f_{Z_a}(y_1) = \norm{\nu^*}_2 M\enspace,   
\end{align*}
which yields
\begin{align*}
 \abs{\cdf_a(r) - \cdf_{\tilde{Z}_a}(r)}  \leq 2\kappa M\enspace,
\end{align*} 
and \ref{lm:cdfbound_2} follows by substituting the definition of $\kappa$.
\end{proof}

\subsection{Proof of Proposition~\ref{prop:patlowerbound}}\label{se:proofpatlowerbound}
\begin{proof}
Recall the definition of \Cref{alg:fairgreedy}. For any $a \in [K]$ let $\hat{r}_{t,a} = \mu_{\tz}^{\top}X_{t,a}$, which is the estimated reward for arm $a$, at round $t$. Note that $\mu_{\tz}$ and $X_{t,a}$ are independent random variables. Furthermore, denote with $\cdf_{\hat{r}_{t,a}}$  the CDF of $\hat{r}_{t,a}$ conditioned on $\mu_{\tz}$, and let
\begin{align*}
    \phi_{t,a} := \cdf_{\hat{r}_{t,a}}(\hat{r}_{t,a} )\enspace, \quad\text{and}\quad \hat \phi_{t,a} := \hat \cdf_{t,a}(\hat{r}_{t,a} )\enspace.
\end{align*}
Now, by the definition of the algorithm, we have
\begin{align*}
    \Pr(a_t = a \,|\, \Hminus_{t-1} ) &= \sum_{m=1}^K \frac{1}{m} \Pr(a \in C_t, |C_t| = m \,|\, \Hminus_{t-1} )\enspace,
\end{align*}
where we introduced $C_t : = \argmax_{a\in [K]} \hat\phi_{t,a}$. Let $\epsilon_t > 0$ and continue the analysis conditioning on the events where $\sup_{a\in[K]}|\phi_{t,a} - \hat \phi_{t,a}| \leq \epsilon_{t}$. Then, we can write
\begin{align*}
   \Pr(a_t = a \,|\, \Hminus_{t-1} ) 
   &\geq  \Pr(\hat \phi_{t,a} > \hat\phi_{t,a'} \,,\,\forall a'\neq a\,|\, \Hminus_{t-1} ) \geq \Pr(\phi_{t,a}  > \phi_{t,a'} + 2 \epsilon_{t}\,,\,\forall~a'\neq a\,|\Hminus_{t-1})\enspace,
\end{align*}
where in the first inequality we considered the case when $a \in C_t$ and $|C_t| = 1$. In the second inequality we considered the worst case scenario where $\hat \phi_{t,a} = \phi_{t,a} -\epsilon_{t}$ and $\hat \phi_{t,a'} = \phi_{t,a'} + \epsilon_{t}$. 
Recall that by the construction of the algorithm $\mu_{\tz} = V_{\tz}^{-1}X_{1:\tz}^\top r_{1:\tz} + (1/\sqrt{d\tz})\cdot\gamma_{\tz}$. for all $a \in [K]$, the additive noise $(1/\sqrt{d\tz})\gamma_{\tz}$ assures that $\mu_{\tz}^\top B_a \neq 0$, almost surely. 
Therefore, by \Cref{lm:prodabs} $\hat r_{t,a} = \scal{\mu_{\tz}}{X_{t,a}}$ conditioned on $\mu_{\tz}$ is absolutely continuous.

\cref{assump}\ref{ass:indep} and \cite[Theorem 2.1.10]{casella2021statistical} yield that $\{\phi_{t,a}\}_{a\in [K]}$ are independent and uniformly distributed on $[0,1]$ and in turn that
\begin{align}\label{eq1:prbound}
  \Pr(a_t = a \,|\, \Hminus_{t-1} )&\geq \int_{0}^{1} \hspace{-.05truecm}\left(\Pr(\phi_{t,a'} < \mu \hspace{-.05truecm}-\hspace{-.05truecm} 2 \epsilon_{t} )\right)^{K-1}\d\mu =\int_{2\epsilon_t}^{1}\hspace{-.05truecm}(\mu{-}2\epsilon_{t})^{K-1}\d\mu =  \frac{(1\hspace{-.05truecm}-\hspace{-.05truecm}2\epsilon_{t})^K}{K}.
\end{align}
We continue by computing an $\epsilon_t$ for which $\sup_{a\in[K]}|\phi_{t,a} - \hat \phi_{t,a}| \leq \epsilon_{t}$ holds with high probability. Observing that, conditioned on $\mu_{\tz}$, $\hcdf_{t, a}$ is the empirical  CDF of $\cdf_{\hat r_{t,a}}$, we can use the Dvoretzky–Kiefer–Wolfowitz-Massart inequality to obtain, for any $a \in [K]$, $t\geq 3$, and $s\geq 0$
\begin{equation*}
    \Pr\left(|\phi_{t,a}- \hat \phi_{t,a}|\geq s\right) \leq 2\exp\left(-2s^2(t-\tz-1)\right)\enspace.
\end{equation*}
Now, let $\tau_0 : = 3 +  8\log^{3/2}\big(5K\e/\delta\big)/\big(1- \sqrt[K]{c}\big)^3$. By applying the union bound, we can write
\begin{align*}
    \Pr\left(\sup_{t\geq \tau_0, a\in[K]}|\phi_{t,a}- \hat \phi_{t,a}|\geq s\right) 
    &\leq K\sum_{t=\tau_0}^{\infty}\Pr\left(|\phi_{t,a}- \hat \phi_{t,a}| \geq s\right) \\ 
    &\leq 2K\sum_{t=\tau_0}^{\infty}\exp\left(-2s^2(t-\tz-1)\right).
   \end{align*}
Since $\tz = \floor{\frac{t-1}{2}}$, it is straightforward to check that
\begin{align*}
    \Pr\left(\sup_{t\geq \tau_0, a\in[K]}|\phi_{t,a}- \hat \phi_{t,a}|\geq s\right)  \leq 2K\int_{t=\tau_0-1}^{\infty}\exp\left(-s^2 t\right)\d t\leq 2Ks^{-2}\exp\left(-s^2(\tau_0-1)\right)\enspace.
\end{align*}
Now, for any $\delta \in (0,1)$, by assigning $s = \sqrt{\frac{\log(4K(\tau_0-1)/\delta)}{\tau_0-1}}$, we get
\begin{align}\label{eq2:prbound}
    \Pr\left(\sup_{t\geq \tau_0, a\in[K]}|\phi_{t,a}- \hat \phi_{t,a}|\geq \sqrt{\frac{\log(4K(\tau_0-1)/\delta)}{\tau_0-1}}\right) \le \frac{\delta}{2\log\left(4K(\tau_0-1)/\delta\right)}\leq \frac{\delta}{4}\enspace,
\end{align}
where from $\tau_0 \geq 3, \delta < 1 \implies 4K(\tau_0-1)/\delta \geq 8 \geq e^2 \implies \log\left(4K(\tau_0-1)/\delta\right)\geq 2$ we obtain the last inequality. From \eqref{eq1:prbound}, it follows that
\begin{align*}
    \inf_{t\geq \tau_0,a\in[K]}\Pr\left(a_t = a |\Hminus_{t-1}\right)\geq \frac{(1-2\sup_{t\geq \tau}\epsilon_{t})^{K}}{K}\enspace.
\end{align*}
Moreover, form \eqref{eq2:prbound}, by letting $\epsilon_t = \sqrt{\frac{\log(4K(\tau_0-1)/\delta)}{\tau_0-1}}$, with probability at least $1-\frac{\delta}{4}$, we have
\begin{align}\label{eq3:prbound}
    \inf_{t\geq \tau, a\in[K]}\Pr\left(a_t = a |\Hminus_{t-1}\right)\geq \frac{1}{K}\left(1-2\sqrt{\underbrace{\frac{\log(4K(\tau_0-1)/\delta)}{\tau_0-1}}_{\text{(I)}}}\right)^{K}\enspace.
\end{align}
For the term (I) in the above, using $\log(x) \leq \log(5\e/4)x^{1/3}$ and $x \geq x^{2/3}$ for any $x \geq 1$ we deduce that\RG{: we can get arbitrarily close to $x^1$ in the denominator, which ultimately will give $K^{2 - q}$ for any $q$ in the final bound.}
\begin{align*}
    \text{(I)}  = \frac{\log(4K/\delta) + \log(\tau_0-1)}{\tau_0-1} \leq  \frac{\log(4K/\delta) + \log(5\e/4)}{(\tau_0-1)^{2/3}} = \frac{\log(5K\e/\delta)}{(\tau_0-1)^{2/3}} \enspace.
\end{align*}
Now, by substituting $\tau_0 =  3 +  8\log^{3/2}\big(5K\e/\delta\big)/\big(1- \sqrt[K]{c}\big)^3$, we get that $\text{(I)} \leq \frac{1}{4}\left(1-\sqrt[K]{c}\right)^{2}$ and conclude the proof by plugging this inequality in \eqref{eq3:prbound}.
\end{proof}

\subsection{Proof of Lemma~\ref{lm:muxbound}}\label{se:proofmuxbound}

We start by establishing some required lemmas.

\begin{lemma}\label{lm:eiguen}
Let $\Sigma$,  $\tau_1$, $\tau_2$ be defined in \Cref{lm:muxbound}, $\tau_3 = \max\left(\tau_1, \tau_2\right)$, $\Sigma = USU^{\top}$ be the compact eigenvalue decomposition of $\Sigma$, with $U\in \mathbb{R}^{d\times r}$, $S \in \mathbb{R}^{r \times r}$ is a diagonal matrix with non-zero diagonal elements, and $U^{\top}U =\mathbbm{I}_r$.
 Denote
$
    \hat{S}_{t_{0}} = \sum_{i=1}^{\tz} U^{\top}X_{i,a_i}X_{i,a_i}^\top  U,
$
  where for $i\in [\tz]$, $a_i$ is given by \Cref{alg:fairgreedy}. Then with probability at least $1-\frac{\delta}{2}$, for any $t\geq 2\tau_3 + 3$ we have
\begin{align*}
    \eigmin\left(\hat{S}_{t_{0}}\right) \geq \frac{(\tz-\tau_3)\eigminplus(\Sigma)}{4}\enspace.
\end{align*}
\end{lemma}
\begin{proof}
Let $\tilde{S}_{t_{0}} := \sum_{i=1}^{\tz}\E\left[U^{\top}X_{i,a_i}X^{\top}_{i,a_i}U|\Hminus_{i-1}\right]$. First, note that for any $\tau_3 \leq i\leq \tz$, we can write 
\begin{align*}
    \E\left[U^{\top}X_{i,a_i}X^{\top}_{i,a_i}U|\Hminus_{i-1}\right] &= \sum_{a=1}^{K}\E\left[U^{\top}X_{i,a_{i}}X^{\top}_{i,a_{i}}U|\Hminus_{i-1}, a_i = a\right]\Pr\left(a_i = a \,|\, \Hminus_{i-1}\right) 
    \\&= \sum_{a=1}^{K}\E\left[U^{\top}X_{i,a}X^{\top}_{i,a}U\right]\Pr\left(a_i = a \,|\, \Hminus_{i-1}\right)\enspace,
\end{align*}
where the last equality holds based on the fact that $X_{i,a_i}X_{i,a_i}^{\top}$ conditioned on $a_i$, is independent from $\Hminus_{i-1}$.
Then, since $t \geq 3 + 64K^3\log^{3/2}\big(5K\e/\delta\big)$, by utilizing \Cref{prop:patlowerbound} with $c = \frac{1}{2}$ and noting that $1/(1-\sqrt[K]{1/2}) \leq 2K$ for all $K \geq 1$, with probability at least $1-\frac{\delta}{4}$, we have $\Pr(a_i = a \,|\,\Hminus_{i-1})\geq \frac{1}{2K}$. Therefore, with probability at least $1-\frac{\delta}{4}$, we obtain
\begin{align*}
    \eigmin\left(\E\left[U^{\top}X_{i,a_i}X^{\top}_{i,a_i}U\,|\,\Hminus_{i-1}\right]\right) \geq \frac{1}{2}\eigmin\left(K^{-1}\sum_{a=1}^{K}U^{\top}\E\left[X_{a}X^{\top}_{a}\right]U\right) = \frac{\eigminplus\left(\Sigma\right)}{2}\enspace,
\end{align*}
and consequently, with probability at least $ 1-\frac{\delta}{4}$ we have
\begin{align}\label{lbound_ad}
    \eigmin(\tilde{S}_{\tz})\geq \sum_{i=1}^{\tz}\eigmin(U^{\top}\E[X_{i,a_i}X_{i,a_i}^{\top} \,|\, \Hminus_{i-1} ]U)\geq (\tz-\tau_3)\cdot\frac{\eigminplus(\Sigma)}{2}\enspace,
\end{align}
where in the last two displays we used the concavity attribute of the function $\eigmin(\cdot)$.
Note that $\left\{X_{i,a_i}\right\}_{i=1}^{\infty}$, is an adaptive sequence with respect to the filtration $\left\{\Hminus_i\right\}_{i=0}^{\infty}$, with 
\begin{align*}
    \norm{U^{\top}X_{i,a_i}X^{\top}_{i,a_i}U}_{\text{op}} \leq \norm{X_{i,a_{i}}}^{2}_{2} \leq \xmax^2\enspace,
\end{align*}
for any $i\in [\tz]$. Let $\iota  = \tz - \tau_3$. Now, by invoking \cite[Theorem 3.1]{mart_random} (with $\delta = \frac{1}{2}$ and $\mu=\eigminplus(\Sigma)/2$, where $\delta, \mu$ are constants that appear in the latter theorem), we have
\begin{align*}
    \Pr\left(\eigmin\left(\hat{S}_{t_{0}}\right) \leq \frac{\iota\eigminplus(\Sigma)}{4} \quad\text{and}\quad \eigmin\left(\tilde{S}_{t_{0}}\right) \geq \frac{\iota\eigminplus(\Sigma)}{2}\right) &
    \leq d\cdot\left(\frac{e^{-\frac{1}{2}}}{{\frac{1}{2}}^{\frac{1}{2}}}\right)^{\frac{\iota\eigminplus(\Sigma)}{4\xmax^2}}
    \leq q\enspace,
\end{align*}
where we introduced $q = d\cdot\exp( -  \frac{\iota\eigminplus\left(\Sigma\right)}{27\xmax^2})$, and we used the inequality $\e^{-\frac{1}{2}}\cdot\frac{1}{2}^{-\frac{1}{2}}\leq \e^{-\frac{4}{27}}$. Note that since $\tz \geq \tau_3 = \frac{54\xmax^2}{\eigminplus(\Sigma)}\log(\frac{4d}{\delta})$, we have $q\leq \frac{\delta}{4}$. Let $p =\Pr[ \eigmin(\tilde{S}_{t_{0}}) \geq \frac{\iota\eigminplus(\Sigma)}{2}]$, then we can write
\begin{align*}
        \Pr\left(\eigmin\left(\hat{S}_{t_{0}}\right) \leq \frac{\iota\eigminplus(\Sigma)}{4} \,\bigg|\, \eigmin\left(\tilde{S}_{t_{0}}\right) \geq \frac{\iota\eigminplus(\Sigma)}{2}\right) &\leq \frac{\delta}{4p}\enspace,
\end{align*}
and accordingly
\begin{align*}
    \Pr\left(\eigmin\left(\hat{S}_{t_{0}}\right) \geq \frac{\iota\eigminplus(\Sigma)}{4} \quad\text{and}\quad \eigmin\left(\tilde{S}_{t_{0}}\right) \geq \frac{\iota\eigminplus(\Sigma)}{2}\right) &\geq 1-\frac{\delta}{2}\enspace,
\end{align*}
where we used $p \geq 1- \frac{\delta}{4}$, which follows from \eqref{lbound_ad}. Substituting $\iota = \tz - \tau_3$ gives the final result. 
\end{proof}
\begin{lemma}\label{lm:boundx}
Let $x \in \cup_{a=1}^{K}{\rm Supp}(X_{a})$ and $\tau_3$ be defined in \Cref{lm:eiguen}, then with probability at least $1-\frac{\delta}{2}$, for all $t\geq 2\tau_3 + 3$ we have

\begin{align*}
    \norm{x}_{V_{\tz}^{-1}} \leq \frac{2\xmax}{
    \sqrt{\eigminplus(\Sigma)(\tz-\tau_3)}}\enspace.
\end{align*}
\end{lemma}
\begin{proof}
Note that if $x=0$ it is straightforward to check that the statement holds. So without loss of generality we assume that $x \in \mathfrak{S}$, where $\mathfrak{S} = \cup_{a=1}^{K}{\rm Supp}(X_{t,a}) - \{0\}$. Consider the compact singular value decomposition $\Sigma = U S U^\top$ where $U \in \R^{d \times r}$, $S \in \R^{r \times r}$ is a diagonal matrix with non-zero diagonal elements (due to \Cref{assump}\ref{ass:abscont}) and $U^\top U = \mathbbm{I}_r$. Denote $\hat{S}_{\tz} = U^{\top}\hat{\Sigma}_{\tz} U$. For any $x \in \mathfrak{S}$ we have from \Cref{suppcov} that $UU^{\top}x = x$, and $x^{\top}UU^{\top} = x^{\top}$. 
First, we claim that
\begin{align*}
    U^{\top}(\hat{\Sigma}_{\tz}+\lambda \mathbbm{I}_d)^{-1}U = (\hat{S}_{\tz}+\lambda \mathbbm{I}_r)^{-1}\enspace.
\end{align*}
To prove the above claim, it is enough to show that 
\begin{align*}
    (\hat{S}_{\tz}+\lambda \mathbbm{I}_r)U^{\top}(\hat{\Sigma}_{\tz}+\lambda \mathbbm{I}_d)^{-1}U = U^{\top}(\hat{\Sigma}_{\tz}+\lambda \mathbbm{I}_d)^{-1}U(\hat{S}_{\tz}+\lambda \mathbbm{I}_r) = \mathbbm{I}_r\enspace.
\end{align*}
Note that 
\begin{align*}
    (\hat{S}_{\tz}+\lambda \mathbbm{I}_r)U^{\top}(\hat{\Sigma}_{\tz}+\lambda \mathbbm{I}_d)^{-1}U&=\big(U^{\top}\hat{\Sigma}_{\tz}U +\lambda \mathbbm{I}_r\big)U^{\top}(\hat{\Sigma}_{\tz}+\lambda \mathbbm{I}_d)^{-1}U
    \\&= U^{\top}\big(\hat{\Sigma}_{\tz}UU^{\top} + \lambda\mathbbm{I}_d\big)\big(\hat{\Sigma}_{\tz}+\lambda \mathbbm{I}_d\big)^{-1}U
    \\&= U^{\top} \big(\hat{\Sigma}_{\tz}+\lambda \mathbbm{I}_d\big)\big(\hat{\Sigma}_{\tz}+\lambda \mathbbm{I}_d\big)^{-1}U = \mathbbm{I}_r\enspace.
\end{align*}
With similar steps one can show that $U^{\top}(\hat{\Sigma}_{\tz}+\lambda \mathbbm{I}_d)^{-1}U(\hat{S}_{\tz}+\lambda \mathbbm{I}_r) = \mathbbm{I}_r$, and therefore $U^{\top}(\hat{\Sigma}_{\tz}+\lambda \mathbbm{I}_d)^{-1}U = (\hat{S}_{\tz}+\lambda \mathbbm{I}_r)^{-1}$.
By exploiting this fact, we can write
\begin{align*}
    \norm{x}_{V_{\tz}^{-1}}^2&= \norm{x}_2^{2}\left( \frac{x}{\norm{x}_2}^{\top}(\hat{\Sigma}_{\tz}+\lambda \mathbbm{I}_d)^{-1}\frac{x}{\norm{x}_2} \right)
    \\&=\norm{x}_2^{2}\left( \frac{x}{\norm{x}_2}^{\top}UU^{\top}(\hat{\Sigma}_{\tz}+\lambda \mathbbm{I}_d)^{-1}UU^{\top}\frac{x}{\norm{x}_2} \right)
    \\&=\norm{x}_2^{2}\left( \frac{x}{\norm{x}_2}^{\top}U(\hat{S}_{\tz}+\lambda \mathbbm{I}_r)^{-1}U^{\top}\frac{x}{\norm{x}_2} \right)
    \\&=\norm{x}_2^{2}\left( \frac{x^{\top}U}{\norm{x^{\top}U}_2}(\hat{S}_{\tz}+\lambda \mathbbm{I}_r)^{-1}\frac{U^{\top}x}{\norm{U^{\top}x}_2} \right) \enspace,
\end{align*}
where the second and last equations are results of \Cref{suppcov}, and consequently 
\begin{align}\label{eq:boundone}
    \norm{x}_{V_{\tz}^{-1}}^2 \leq \frac{\xmax^2}{\lambda_{\min}(\hat{S}_{\tz})}\enspace.
\end{align}
On the other hand, from \Cref{lm:eiguen}, with probability at least $1- \frac{\delta}{2}$, we have
\begin{align}\label{eq:boundtwo}
    \eigmin\left(\hat{S}_{t_{0}}\right) \geq \frac{(\tz-\tau_3)\eigminplus(\Sigma)}{4}\enspace.
\end{align}
Finally, by combining \eqref{eq:boundone} and \eqref{eq:boundtwo}  with probability at least $1- \frac{\delta}{2}$ we have
\begin{align*}
     \norm{x}_{V_{\tz}^{-1}}^2 \leq \frac{4\xmax^2}{
    \eigminplus(\Sigma)(\tz - \tau_3)}\enspace.
\end{align*}
\end{proof}

\begin{lemma}\label{lm:boundmu}
With probability at least $1 - \frac{\delta}{4}$, for all $t \geq 3$ we have
\begin{align*}
\norm{\mu^* - \mu_{\tz}}_{V_{\tz}} \leq (\lambda^{\frac{1}{2}}+R + \xmax)\sqrt{d\log((8+8\tz\max(\xmax^2/\lambda,1))/\delta)} +\lambda^{\frac{1}{2}}\norm{\mu^*}_2\enspace.
\end{align*}
\end{lemma}
\begin{proof}
    Recall that by the definition of \Cref{alg:fairgreedy}, we have $\mu_{\tz} = V_{\tz}^{-1}X_{1:\tz}^\top r_{1:\tz} + (1/d\sqrt{\tz})\cdot\gamma_{\tz}$. Therefore, we can write
    \begin{align*}
        \norm{\mu^* - \mu_{\tz}}_{V_{\tz}} \leq \underbrace{\norm{\mu^* - V_{\tz}^{-1}X_{1:\tz}^\top r_{1:\tz}}_{V_{\tz}}}_{\text{(I)}} + \underbrace{\frac{\rho}{d\sqrt{\tz}}\norm{\gamma_{\tz}}_{V_{\tz}}}_{\text{(II)}}\enspace.
    \end{align*}
We proceed the proof by providing upper bounds for (I) and (II). For (I), by invoking \cite[Theroem 2]{abbasi2011improved}, with probability at least $1- \frac{\delta}{8}$, for all $t \geq 3$, which implies $\tz \geq 1$ we have
\begin{align*}
    \text{(I)} &\leq R\sqrt{d\log((8+8\tz\xmax^2/\lambda)/\delta)} + \lambda^{\frac{1}{2}}\norm{\mu^*}_2 \\
    &\leq R\sqrt{d\log((8+8\tz\max(\xmax^2/\lambda,1))/\delta)} + \lambda^{\frac{1}{2}}\norm{\mu^*}_2\enspace.
\end{align*}
On the other hand, since $\rho \leq 1$, for term (II) we have
\begin{align*}
     \text{(II)} \leq \frac{1}{d\sqrt{\tz}}\norm{V_{\tz}}_{\op}^{\frac{1}{2}}\norm{\gamma_{\tz}}_2 \leq \frac{\xmax+\lambda^{\frac{1}{2}}}{d}\norm{\gamma_{\tz}}_2\enspace.
\end{align*}
\begin{align*}
    \Pr\big( \text{(II)}\geq (\xmax+\lambda^{\frac{1}{2}})\sqrt{\log(8d/\delta})\big) &\leq \Pr\big(\norm{\gamma_{\tz}}_2\geq d\sqrt{\log(8d/\delta})\big)\\
    &\leq d\Pr\big(|\gamma_{1,\tz}|\geq \sqrt{\log(8d/\delta})\big)\leq \frac{\delta}{8}\enspace.
\end{align*}
Thus, by applying the union bound with probability at least $1- \frac{\delta}{8}$, for all $t\geq 3$ we have
\begin{align*}
    \text{(II)}\leq (\xmax+\lambda^{\frac{1}{2}})\sqrt{\log(8\tz d/\delta)}\leq (\xmax+\lambda^{\frac{1}{2}})\sqrt{d\log((8+8\tz\max(\xmax^2/\lambda,1))/\delta)} \enspace.
\end{align*}
\end{proof}

\begin{proof}[Proof of \Cref{lm:muxbound}]
Recall that $\tau_3 = \max (\tau_1, \tau_2)$. From \Cref{lm:boundx}, with probability at least $1-\frac{\delta}{2}$ for all $t \geq 2\tau_3 + 3$ we have
\begin{align*}
    \norm{x}_{V_{\tz}^{-1}} \leq \frac{2\xmax}{
    \sqrt{\eigminplus(\Sigma)(\tz-\tau_3)}}\enspace.
\end{align*}
From \Cref{lm:boundmu}, with probability at least $1-\frac{\delta}{4}$ for all $t \geq 3$
\begin{align*}
\norm{\mu^* - \mu_{\tz}}_{V_{\tz}} \leq (\lambda^{\frac{1}{2}}+R + \xmax)\sqrt{d\log((8+8\tz\max(\xmax^2/\lambda,1))/\delta)} +\lambda^{\frac{1}{2}}\norm{\mu^*}_2\enspace.
\end{align*}
Thus, combining \Cref{,lm:boundx,lm:boundmu}, with probability at least $1-\frac{3\delta}{4}$ for all $t\geq 2\tau_3 + 3$ we have
\begin{align*}
    &\norm{\mu^* - \mu_{\tz}}_{V_{\tz}}\norm{x}_{V_t^{-1}} \leq \\ &\qquad\frac{2\xmax}{\sqrt{\eigminplus(\Sigma)(\tz - \tau_3)}}\left((\lambda^{\frac{1}{2}}+R + \xmax)\sqrt{d\log((8+8\tz\max(\xmax^2/\lambda,1))/\delta)} +\lambda^{\frac{1}{2}}\norm{\mu^*}_2\right)\enspace.
\end{align*}
By the fact that $t \geq 4\tau_3 + 3$, we have $\tz \geq 2\tau_3$, which implies $\frac{1}{\sqrt{\tz -\tau_3}}\leq \sqrt{\frac{2}{\tz}}$. We conclude the proof by using the inequality $\tz \geq \frac{t-3}{2} \geq \frac{t}{8}$, for all $t \geq 4$.
\end{proof}

\subsection{Proof of Theorem~\ref{th:regret}}\label{se:proofregret}

\begin{proof}

Combining \Cref{lm:instantbound} with \Cref{lm:muxbound} and using $1/(t-1) \leq 3/(4t)$ for all $t \geq 4$ we obtain, with probability at least $1-\delta$ and for all $\tau \leq t \leq T$
\begin{align*}
    \cdf_{a_t^*}(&\scal{\mu^*}{X_{t,a_t^*}}) -\cdf_{a_t}(\scal{\mu^*}{X_{t, a_t}}) \leq 4\sqrt{\frac{\log(8KT/\delta)}{3t}} \\&\quad\quad +\frac{48M\xmax}{\sqrt{\eigminplus(\Sigma)t}}\left((\lambda^{\frac{1}{2}}+R + \xmax)\sqrt{d\log((8+4t\max(\xmax^2/\lambda,1))/\delta)} +\lambda^{\frac{1}{2}}\norm{\mu^*}_2\right) \enspace.
\end{align*}
By summing up the last inequality, with probability at least $1-\delta$ we get 
\begin{equation}
\begin{aligned}\label{th:eq1}
    \sum_{t=\tau}^{T}\bigg[\cdf_{a_t^*}(&\scal{\mu^*}{X_{t, a_t^*}}) -  \cdf_{a_t}(\scal{\mu^*}{X_{t, a_t}})\bigg] \leq 8\sqrt{\frac{T\log(8KT/\delta)}{3}}\\& \hspace{-.5truecm} +\frac{96M\xmax}{\sqrt{\eigminplus(\Sigma)}}\left((\lambda^{\frac{1}{2}}+R + \xmax)\sqrt{dT\log((8+4T\max(\xmax^2/\lambda,1))/\delta)} +\sqrt{\lambda T}\norm{\mu^*}_2\right).
\end{aligned}
\end{equation}
where the last display is obtained by the inequality $\sum_{t=1}^{T}t^{-\frac{1}{2}}\leq 2T^{\frac{1}{2}}$. On the other hand, for $t\in[T]$, $\cdf_{a_t^*}(\scal{\mu^*}{X_{t, a_t^*}}) -  \cdf_{a_t}(\scal{\mu^*}{X_{t, a_t}})\leq 1$, and we can write
\begin{align}\label{th:eq2}
    \sum_{t=1}^{\tau}\bigg[\cdf_{a_t^*}(\scal{\mu^*}{X_{t, a_t^*}}) &-  \cdf_{a_t}(\scal{\mu^*}{X_{t, a_t}})\bigg] \leq \tau\enspace.
\end{align}
By combining \eqref{th:eq1} and \eqref{th:eq2}, we conclude the proof.
\end{proof}

\section{Experiments}\label{se:expadditional}

In this section we include additional details on the simulation experiments in \Cref{se:simu} and an experiment on the US census data.

\subsection{Additional Details on the Simulation}
We use the following value for the underlying linear  model used in \Cref{fig:all}.
\begin{align*}
    \mu^* = (&\underbrace{4, 3, 7, 0}_{\text{Group 1}}, 
   \underbrace{8, 0, 0, 0}_{\text{Group 2}}, 
   \underbrace{5, 5, 0, 0}_{\text{Group 3}}, 
   \underbrace{2, 2, 2, 2}_{\text{Group 4}}, 1)\enspace.
\end{align*}
Each slice of 4 coordinates of $\mu^*$ affects a different group. Furthermore, since each coordinate of $Y_a$ follows a standard uniform distribution, the resulting reward distributions for each group follow weighted variants of the Irwin-Hall distribution \citep{hall1927distribution}.

\subsection{Experiments on US Census data}\label{se:censusone}

\begin{figure}[t!]
     \centering
     \begin{subfigure}[b]{0.49\textwidth}
         \centering
\includegraphics[width=\textwidth]{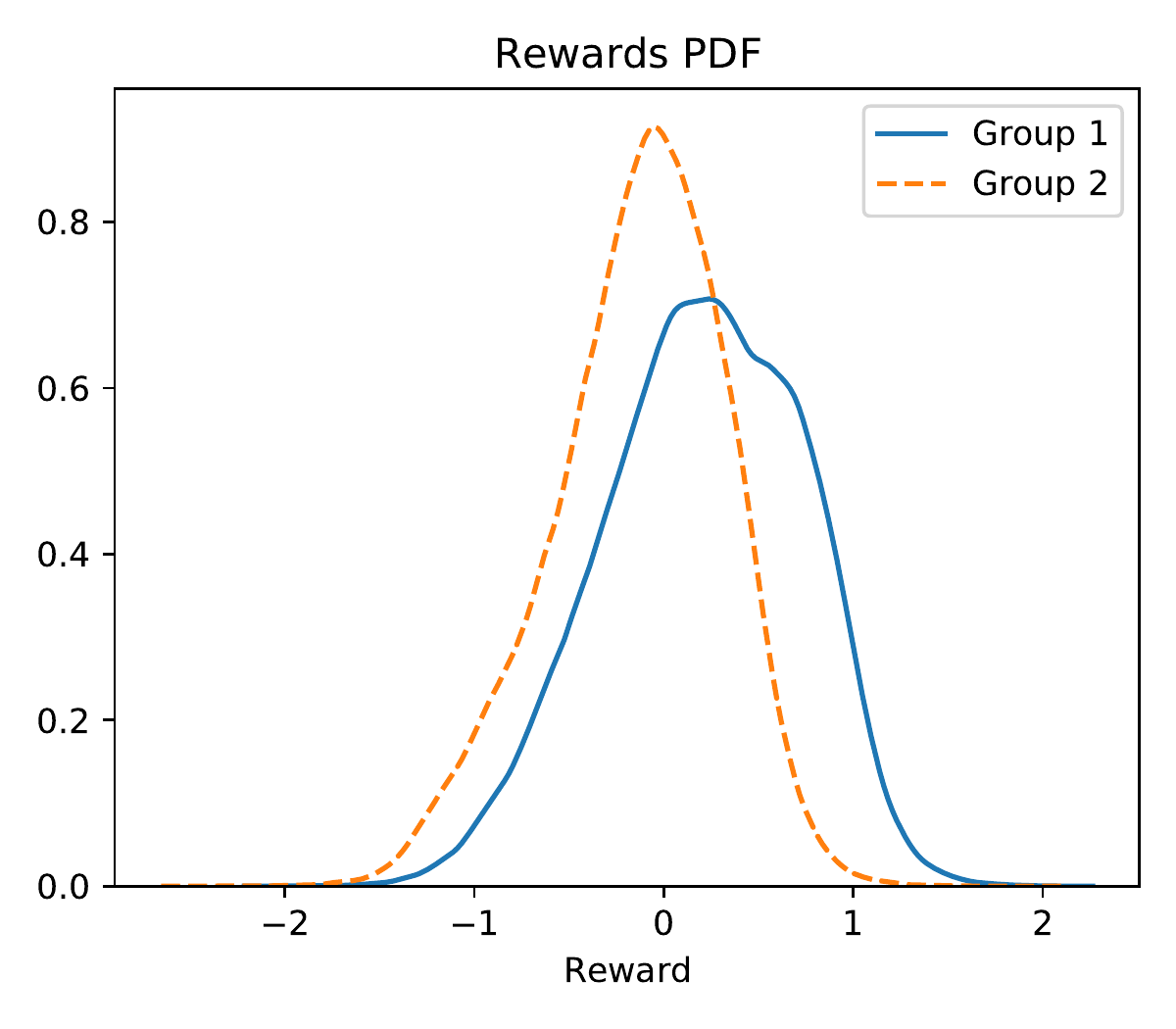}
     \end{subfigure}
     \hfill
     \begin{subfigure}[b]{0.49\textwidth}
         \centering
         \includegraphics[width=\textwidth]{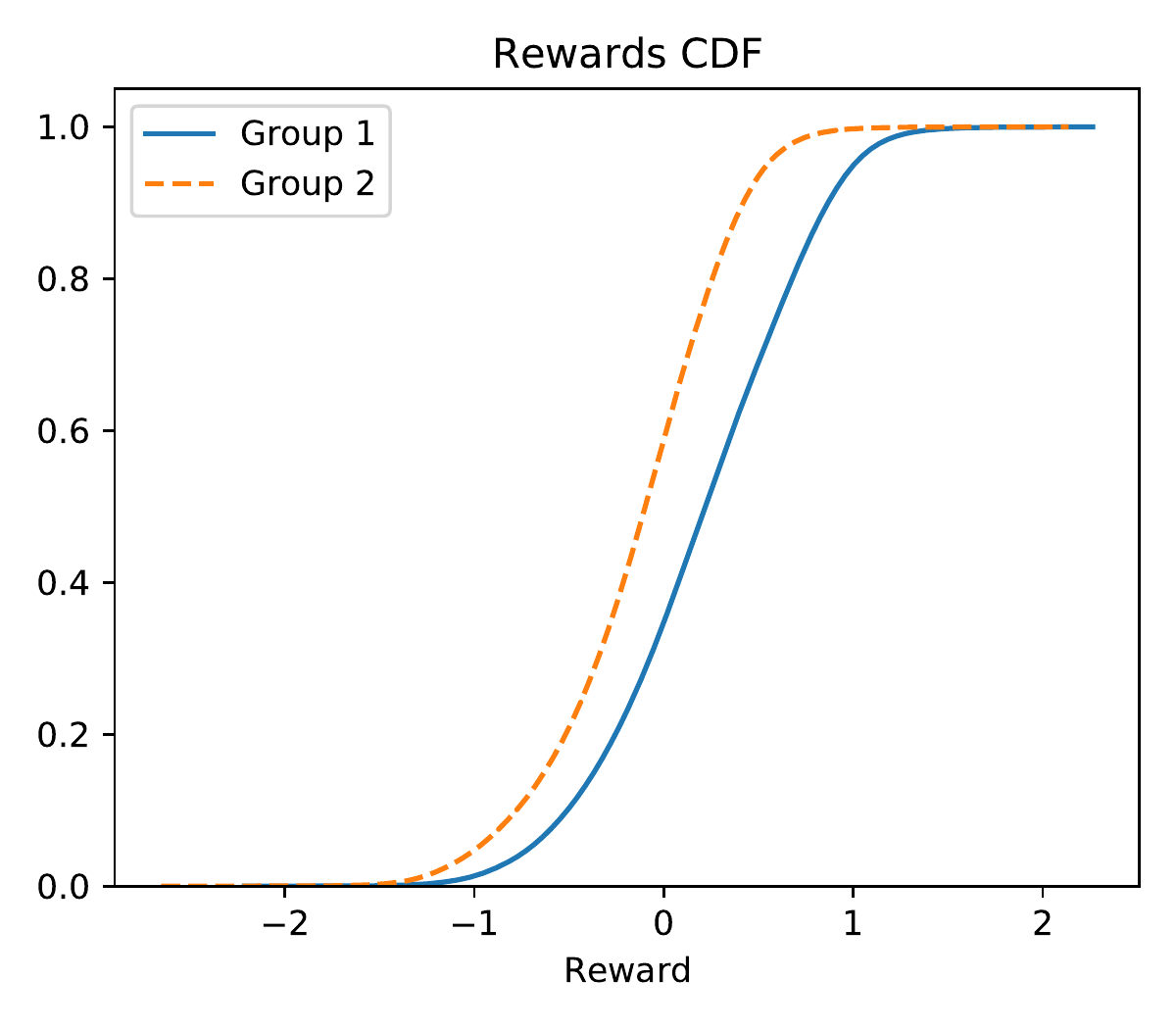}
     \end{subfigure}
     \hfill
    \begin{subfigure}[b]{0.49\textwidth}
    \centering
    \includegraphics[width=\textwidth]{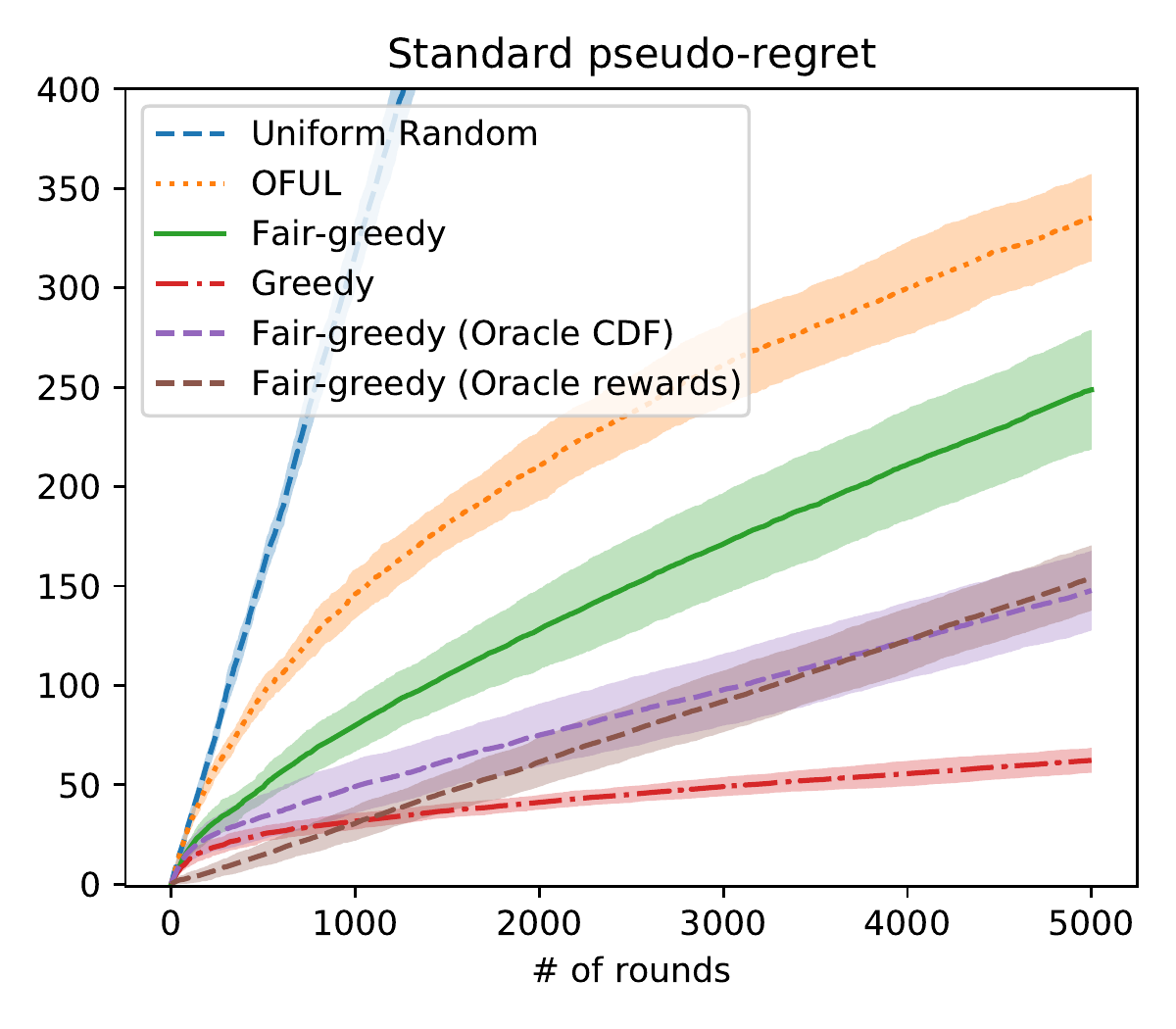}
     \end{subfigure}
     \hfill
     \begin{subfigure}[b]{0.49\textwidth}
         \centering
         \includegraphics[width=\textwidth]{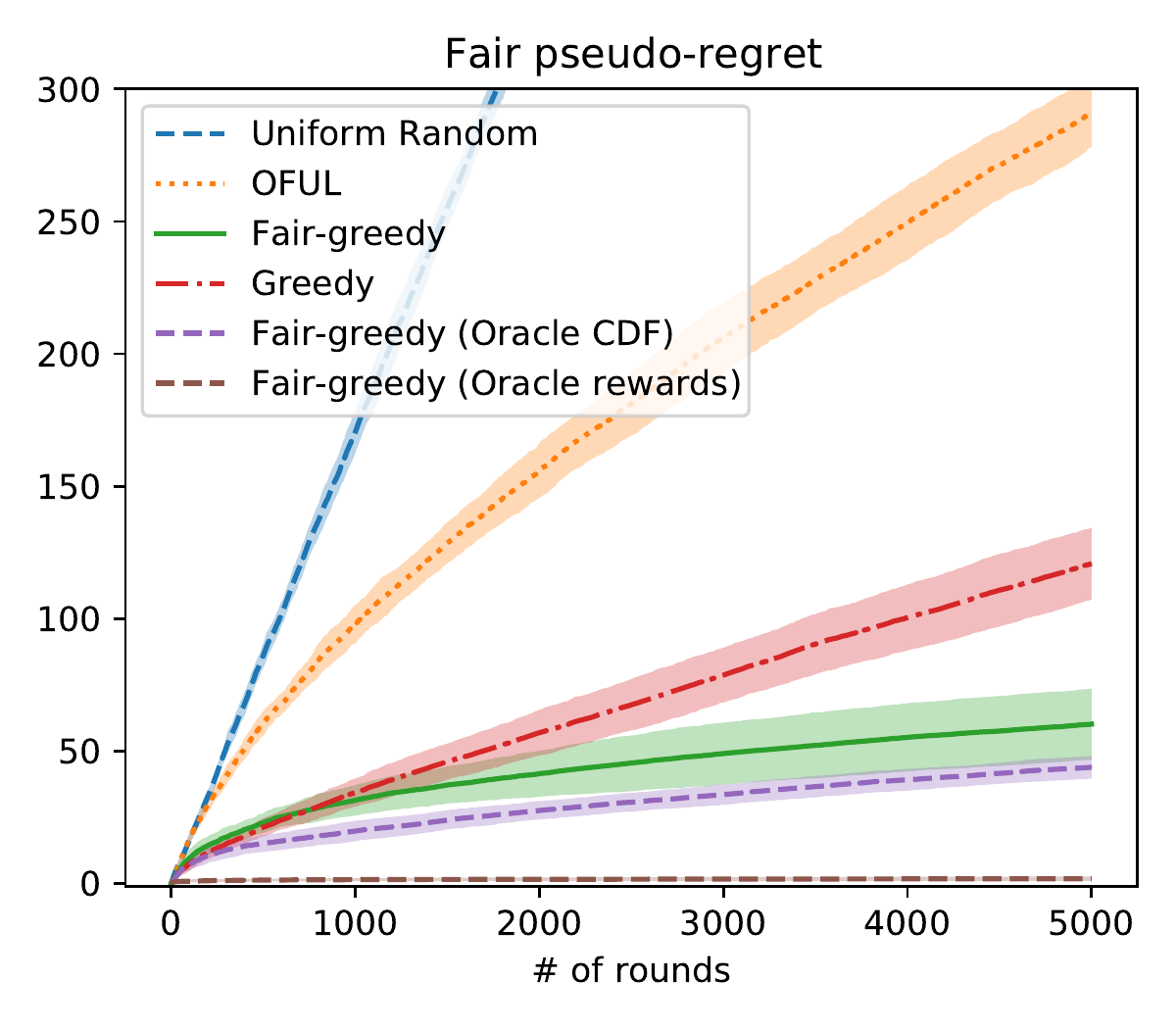}
     \end{subfigure}
        \caption{\small \textbf{US Census Results}. 
        Top two images are density and CDF plots of the reward distributions while the bottom two plots are the standard and fair pseudo-regrets, with mean (solid lines) $\pm$ standard deviation (shaded region) over 10 runs. To compute the reward PDF and CDF for each group we use the empirical CDF on all $500K$ samples from $D2$.
        }
        \label{fig:expcensus}
\end{figure}

In this section, we present an experiment performed using the US Census data and the FalkTables library\footnote{\url{https://github.com/zykls/folktables}} \cite{ding2021retiring}. In particular we construct a dataset with features similar to the UCI Adult dataset but where the target is the person's income instead of the binary variable indicating if the income is more or less than $50K$ dollars. We use this target as a possibly inaccurate proxy for how well a candidate will perform on the job, hence it is used as the noisy reward for the bandit problem.

\textbf{Setup and Preprocessing.} To setup the bandit problem, we construct 2 datasets, namely $D1$ and $D2$, by selecting $500K$ males and $500K$ females random samples first  from the $2017$ US Census Survey, to assemble $D1$, and then from the $2018$ survey to assemble $D2$. We use $D1$ to find mean and standard deviation for each feature and also for the target. After that we normalize features and target from $D2$ by subtracting the mean and dividing by the standard deviation previously computed on $D1$. We then construct $\mu^*$ as a ridge regression estimate on the samples from $D2$ with the regularization parameter equal to $10^{-8}$. The regression vector $\mu^*$ will be used to compute the (true) rewards for the samples. We construct the bandit problem with $K=2$ arms/groups which correspond to the gender identities male and female. At each round, the context vectors of one male and one female candidate are sampled from $D2$ and after one of the two is selected by the policy, its corresponding noisy reward (i.e.\@ its income) is received by the agent. 

\textbf{Baselines.} We compare our method, namely \textit{Fair-greedy} (\Cref{alg:fairgreedy}), with the following baselines. 
\begin{itemize}
    \item \textit{Uniform Random}, which selects an arm uniformly at random at each round. 
    \item \textit{OFUL} \cite{abbasi2011improved}, with exploration parameter set to $0.1$.
    \RG{Leo do you think we should omit this or write it better?}
    \item \textit{Greedy}, which computes the ridge regression estimate for the reward vector using all the selected contexts and noisy rewards in the history and then selects the arm maximising the estimated reward.
    \item  \textit{Fair-greedy (Oracle CDF)}, which is a variant of \textit{Fair-greedy} where all the selected contexts and noisy rewards in the history are used to compute the ridge regression estimate and the empirical CDF of each group is replaced by the true CDF.
    \item  \textit{Fair-greedy (Oracle rewards)}, which is another variant of \textit{Fair-greedy} where the ridge regression estimate is replaced by the true reward model $\mu^*$ and all contexts in the history are used to compute the empirical CDF for each group.
\end{itemize}
Note that the last two methods are \textit{oracle methods} because they rely either on the true CDF of the rewards for each group or on $\mu^*$, which are unknown to the agent.
All methods using a ridge regression estimate have the regularization parameter set to $0.1$. We observed that varying this parameter did not affect much the relative performance of the methods.

\textbf{Results.} The results and the reward distributions are illustrated in \Cref{fig:expcensus}. We note that in this case, \textit{Greedy} performs much better than OFUL, which appears to be too conservative for this problem.  
In particular, the standard pseudo-regret of \textit{Greedy} is unrivaled after $1000$ rounds. Furthermore, since there is a large overlap in the distributions of rewards, our \textit{Fair-greedy} policy performs much better than the \textit{Uniform Random} policy even in terms of standard pseudo-regret, while it outperforms all non-oracle methods in terms of fair pseudo-regret. As expected, the oracle methods both achieve a lower fair pseudo-regret than \textit{Fair-greedy}, and we note that knowing only the underlying model $\mu^*$ is significantly more advantageous than knowing only the CDF for each group.

\section{Multiple Candidates for Each Group}\label{se:multtiple}

This section contains a rigorous treatement of the content in \Cref{se:multiplemain}. We consider the more realistic case where contexts from a given arm do not necessarily belong to the same group.
 In particular, we assume that at each round $t$, the agent receives tuples $\{(X_{t,a}, s_{t,a})\}_{a=1}^{K}$, where $s_{t,a} \in [G]$ is the sensitive group of the context $X_{t,a} \in \R^d$ and $G$ is the total number of groups. After that the agent selects action $a_t$ and subsequently receives the noisy reward $\scal{\mu^*}{X_{t,a}} + \eta_t$ 

Note that we recover the original setting discussed in \Cref{sec:problem} when $G=K$ and  $s_{t,a} = a$ for every $a \in [K]$, $t \in \N$. A more realistic scenario is when $\{(X_{t,a}, s_{t,a})\}_{a=1}^K$ are i.i.d., and the distribution represents e.g.\@ the underlying  population of candidates, where $\Pr(s_{t,a} = i)$ is the same for all $a \in [K]$ and can be small when the group $i$ is a minority. 
The following analysis applies to both cases.

We impose the following assumption, which is a natural extension of \Cref{assump}.
\begin{assumption}\label{ass:b}
Let $\mu^* \in \R^d$ be the underlying reward model.
We assume that:
\begin{enumerate}[label={\rm (\roman*)}]
    \item\label{ass:noise_2} The noise random variable $\eta_t$ is zero mean $\noisevar$-subgaussian, conditioned on $\mathcal{H}_{t-1}$. 
    
    \item\label{ass:iid_2} For any $a \in [K]$, let $(X_a, s_a)$ be a random variable with values in $\R^d\times [G]$ and  $\norm{X_a}_2 \leq \xmax$ almost surely. $\{(X_{t,a}, s_{t,a})\}_{t=1}^{T}$ are i.i.d. copies of $(X_a, s_a)$. $X_a$ conditioned to $s_a=i$ is a copy of the random variable $\hat X_i$ which is independent on the arm, for every $a \in [K]$.

    \item\label{ass:indep_2} For every $a \in [K]$ $X_a$  conditioned to $s_a$ is independent from $(X_{a'}, s_{a'})$ for any $a' \neq a$. 
    \item\label{ass:abscont_2} 
    For every $i \in [G]$, then there exist $d_i \geq 1$, an absolutely continuous random variable $Y_i$ with values in $\R^{d_i}$ admitting a density $\pdf_{i}$,  $B_i \in \R^{d \times d_i}$ and $c_i \in \R^d$ such that $B_i^\top B_i = \mathbbm{I}_{d_i}$,
    \begin{equation*}
        \hat X_i  =  B_i Y_i + c_i \quad\text{and}\quad \mu^{*\top} B_i \neq 0\enspace.
    \end{equation*}
\end{enumerate}
\end{assumption}

We define $\cdf(r, i) = \Pr(\scal{\mu^*}{\hat X_i} \leq r ) = \Pr(\scal{\mu^*}{X_a} \leq r \,|\, s_a = i)$ for any $r \in \R$, $i \in G$. Hence we can extend the definition of group meritocratic fairness as follows.
\begin{definition}[GMF policy]
a policy $\{a_t^*\}_{t=1}^{\infty}$ is group meritocratic fair (GMF) if for all $t \in \N, a \in [K]$ it satisfies
\begin{equation*}\label{eq:gmf}
     \cdf(\scal{\mu^*}{X_{t, a_t^*}}, s_{a_t^*}) \geq \cdf(\scal{\mu^*}{X_{t, a}}, s_{t,a_t})\enspace.
\end{equation*}
\end{definition}

The fair pseudo-regret is now defined as
\begin{equation}\label{eq:fairregret_2}
R_F(T) = \sum_{t=1}^T \cdf(\scal{\mu^*}{X_{t,a^*_t}}, s_{a^*_t}) - \cdf(\scal{\mu^*}{X_{t,a_t}}, s_{a_t})
\end{equation}

We can adapt \Cref{pr:gmdm} to this setting as follows.

\begin{proposition}[GMF policy satisfies \textit{history-agnostic demographic parity}]\label{pr:gmdm_2}
Let $\{\scal{\mu^*}{X_a}\}_{a=1}^{K}$ conditioned to $\{s_a\}_{a=1}^{K}$ be independent and absolutely continuous and for every $a\in[K], t \in \N$, let $(X_{t,a}, s_{t,a})$ be an i.i.d. copy of $(X_a, s_a)$. Then for every $t \in \N$, $\{\cdf(\scal{\mu^*}{X_{t,a}}, s_{t,a})\}_{a=1}^K$ conditioned to $\{s_{t,a}\}_{a=1}^{K}$ are i.i.d. uniform on $[0,1]$ and 
\begin{equation}\label{eq:gmdm_2}
    \Pr (a_t^* = a \,|\, s_{t,a}, \Hminus_{t-1} ) = 1/K \qquad \forall a
    \in [K],
\end{equation}
for any GMF policy $\{a^*_t\}_{t=1}^{\infty}$.
\end{proposition}
\begin{proof}
Let $\psi_a : = \cdf(\scal{\mu^*}{X_{t,a}}, s_{t,a})$.
From the assumptions $\{\psi_a\}_{a=1}^{K}$ conditioned to $\{s_{t,a}\}_{a=1}^{K}$ are i.i.d random variables, independent from $\Hminus_{t-1}$, with uniform distribution on $[0,1]$ (see \cite[Theorem 2.1.10]{casella2021statistical}). Let $\tilde \Pr = \Pr(\cdot \,|\, \{s_{t,a}\}_{a=1}^{K}, \Hminus_{t-1})$, we have that $\forall a_1, a_2 \in [K]$: $\tilde\Pr(\psi_{a_1} = \psi_{a_2}) = 0$,  $\tilde\Pr (a_t^* = a \,|\, \Hminus_{t-1}) = \tilde\Pr(a_t^* = a)$ and
\begin{equation*}
\tilde\Pr(a_t^* = a_1 ) = \tilde\Pr(\psi_{a_1} > \psi_{a'}, \, \forall a' \neq a_1) = \tilde\Pr(\psi_{a_2} > \psi_{a'}, \, \forall a' \neq a_2) =  \tilde\Pr (a_t^* = a_2) = 1/K\enspace.
\end{equation*}
Let $S_{t, \not a} = \{s_{t,a} : a \in [K]/{a} \}$, then the statement follows from
\begin{equation*}
    \Pr (a_t^* = a \,|\, s_{t,a}, \Hminus_{t-1} ) = \E_{S_{t, \not a}}[\tilde\Pr (a_t^* = a)]\enspace.
\end{equation*}
\end{proof}

\Cref{pr:gmdm_2} states that the probability of selecting an arm does not change based on group membership. Fair-Greedy V2 in \Cref{alg:fairgreedy_2} is the extension of the fair-greedy policy to this new setting. 

\begin{algorithm}[ht!]
\caption{Fair-Greedy V2}\label{alg:fairgreedy_2}
\begin{algorithmic}[1]
\State \textbf{Requires} regularization parameter $\lambda>0$, and noise magnitude $\rho \in (0, 1]$
\For{$t=1 \dots T$}
\State\label{empty1_2} Receive  $\{(X_{t,a}, s_{t,a})\}_{a=1}^{K}$
\State\label{step:assemble_2} Set $\tz = \floor{(t-1)/2}$,  $X_{1:\tz} =  (X_{1,a_1}, \dots, X_{\tz,a_{\tz}})^\top $,  $r_{1:\tz} = (r_{1,a_1}, \dots, r_{\tz,a_{\tz}})$.
\State\label{step:mu_2} \textbf{If} $\tz = 0$ set $\mu_{\tz} = 0$, \textbf{else} let $V_{\tz} := X_{1:\tz}^\top X_{1:\tz} + \lambda \mathbbm{I}_d$, generate
$\gamma_{\tz} \sim \mathcal{N}(0, \mathbbm{I}_d)$ and compute 
\begin{equation*}
    \mu_{\tz} := V_{\tz}^{-1}X_{1:\tz}^\top r_{1:\tz} + \frac{\rho}{d\sqrt{\tz}}\cdot\gamma_{\tz}\enspace.
\end{equation*}
\State\label{step:ecdf_2} For each $a \in [K]$, let $i :=s_{t,a}$ and $N_{t,i} =  \sum_{j=\tz+1}^{t-1}\sum_{a'=1}^{K} \indic{s_{j,a'}=i}$, compute 
\begin{equation*}
\hcdf_t (\scal{\mu_{\tz}}{ X_{t,a}}, i) :=  N_{t,i}^{-1} \sum_{j=\tz+1}^{t-1} \sum_{a'=1}^{K} \indic{ \scal{\mu_{\tz}}{ X_{j,a'}}  \leq \scal{\mu_{\tz}}{X_{t,a}}} \indic{s_{j,a'}=i}\enspace.
\end{equation*}

\State\label{step:at_2} Sample action
\begin{equation*}
  a_t \sim \mathcal{U}\big[\argmax_{a\in [K]} \hcdf_t (\scal{\mu_{\tz}}{X_{t,a}}, s_{t,a})\big]\enspace.
\end{equation*}
\State\label{empty2_2} Observe noisy reward $r_{t,a_t} = \scal{\mu}{X_{t,a_t}} + \eta_t$.

\EndFor
\end{algorithmic}
\end{algorithm}

Notice that the number of contexts used for the CDF approximation for group $i \in [G]$ is now the random variable $N_{t,i}$. Furthermore, we are now using contexts from all the arms to estimate the CDFs, which as we will see it can improve the dependency on $K$ in the fair pseudo-regret bound.
We observe that the information averaged demographic parity property of \Cref{DGP} does not transfer directly to Fair-Greedy V2, because at each round, there can be a different number of candidates for each group. However, as we will see, the regret is still similar to the original case.

The following Lemma establishes an high probability lower bound on $N_{t,i}$.

\begin{lemma}\label{lm:chernoff}
Let $q_K : = \min_{i \in [G] } \sum_{a=1}^{K} \Pr(s_a = i)$  and let 
\begin{equation*}
    \isrand = \indic{ \exists a \in [K] \text{ such that } \forall i \in [G]  \, \Pr(s_a = i) < 1}\enspace.
\end{equation*}
$\isrand=1$ means that the sensitive attribute is random for at least one arm, while is deterministic if $\isrand=0$. 
Then, let $\alpha = \isrand b + (1-\isrand)$ with $b \in (0,1)$ and $t_N : = 3 +\isrand\ceil{\frac{2}{(1-\alpha)^2 q_K}\log(GT/\delta)}$,  with $N_{t,i}$ defined at \Cref{step:ecdf_2} of \Cref{alg:fairgreedy_2} and,without loss of generality, $q_K > 0$. For simiplicity we let $\isrand x = 0$ when $\isrand =0, x = \infty$. We have that with probablity at least $1-\isrand\delta$, for every $ t \in  \{t_N, \dots,T\}$ 
\begin{equation*}
    \min_{i \in [G]} N_{t,i} \geq (t-1-\tz) \alpha q_K
\end{equation*}
\end{lemma}
\begin{proof}
If $\isrand = 0$, then  $\Pr(s_a = i) = \indic{s_a = i}$ and the result follows.

If $\isrand = 1$ instead, note that for every $i \in [G]$ we have that
\begin{equation*}
\E[N_{t,i}] = \sum_{j=\tz-1}^{t-1}\sum_{a=1}^K \Pr(s_a = i) \geq (t-1-\tz)q_K\enspace.
\end{equation*}

Applying the Chernoff bound we have that with probability at least $1-\delta$, for all $t > t_N$ 
\begin{equation*}
    N_{t,i} \geq \alpha \E[N_{t,i}] \geq (t-1-\tz)\alpha q_K\enspace,
\end{equation*}
and the statement follows
\end{proof}

Let $ S_T(t_N, \alpha) := \big\{\{\{s_{t,a}\}_{a=1}^{K}\}_{t=1}^{T} \,:\,  \min_i N_{t,i} \geq (t-\tz-1)\alpha q_K \, \text{for all} \,t_N \leq t \leq T  \big\}$ be the event when \Cref{lm:chernoff} is satisfied. We can then proceed the analysis assuming that $S_T(t_N, \alpha)$ holds. Noticing that the maximum number of approximate CDFs to be computed at each round is $G$ we can adapt \Cref{lm:dwk} as follows.

\begin{lemma}\label{lm:dwk_2}
Let \Cref{ass:b}\ref{ass:iid_2} hold and $\hcdf_t(r, i)$, and $\mu_{\tz}$ to be generated by \Cref{alg:fairgreedy_2}. Let $Z_i : = \scal{\mu_{\tz}}{\hat X_i}$ and denote with $\cdf_{Z_i}(\cdot)$ its CDF, conditioned on $\mu_{\tz}$. Then, if the event $S_T(t_N, \alpha)$ is satisfied, with probability at least $1-\delta$ we have that for all $t_N \leq t \leq T$
 \begin{equation*}
     \sup_{i \in [G], r \in \R} \abs{\hcdf_{t}(r,i)-\cdf_{Z_i}(r)} \leq \sqrt{\frac{\log(2GT/\delta)}{(t-1)\alpha q_K}}\enspace.
 \end{equation*}
\end{lemma}

Then, following the  steps  in \Cref{lm:instantbound}, we obtain the following bound on the instantaneous regret.

\begin{lemma}[Instant regret bound]
\label{lm:instantbound_2}
Let \Cref{ass:b}\ref{ass:iid_2}\ref{ass:abscont_2} hold and $a_t$ to be generated by \Cref{alg:fairgreedy_2}. Then, if the event $S_T(t_N, \alpha)$ is satisfied, with probability at least $1-\delta/4$, for all $t$ such that $t_N \leq t \leq T$ we have
\begin{align*}
    \cdf(&\scal{\mu^*}{X_{t,a_t^*}}, s_{a_t^*}) -\cdf(\scal{\mu^*}{X_{t, a_t}}, s_{a_t}) \leq 6M\norm{\mu^* - \mu_{\tz}}_{V_{\tz}}\norm{x_{\max}}_{V^{-1}_{\tz}} + 2\sqrt{\frac{\log(16GT/\delta)}{(t-1)\alpha q_K}}\enspace,
\end{align*}
where $\norm{x_{\max}}_{V^{-1}_{\tz}} := \sup_{x \in \cup_{i=1}^{G}{\rm Supp}(\hat X_i)}\norm{x}_{V^{-1}_{\tz}}$.
\end{lemma}
\begin{proof}[Proof sketch]
    Uses the decomposition in the proof of \Cref{lm:instantbound}, then \Cref{lm:dwk_2} and a version of \Cref{lm:cdfbound} adapted to this more general setting.
\end{proof}
We can bound $\norm{\mu^* - \mu_{\tz}}_{V_{\tz}}$ using the confidence bounds in OFUL \citep{abbasi2011improved}.
To bound $\norm{x_{\max}}_{V^{-1}_{\tz}}$ instead, we first provide an adaptation  of \Cref{prop:patlowerbound}, which guarantees sufficient exploration of all arms. The proof is very similar to that of \Cref{prop:patlowerbound} and we report it here for completeness. 
\begin{proposition}\label{prop:patlowerbound_2}
Let \Cref{assump} hold, $a_t$ be generated by \Cref{alg:fairgreedy} and $c \in [0,1)$. Then, if $S_T(t_N, \alpha)$ is satisfied, with probability at least $1-\delta/4\,$, for all $a \in [K]$ and all $t \geq  \max\big(t_N, 3 + 8\log^{3/2}\big(5G\e/\delta\big)\big(1- \sqrt[K]{c}\big)^{-3}(q_K \alpha)^{-3/2}\big)$ we have
\begin{align*}
  \Pr(a_t = a \,|\, s_{t,a}, \Hminus_{t-1} ) &\geq \frac{c}{K} \,\enspace,
\end{align*}
where we recall that  $\Hminus_t = \cup_{i=1}^{t} \left\{\{(X_{i, a}, s_{i,a})\}_{a=1}^K, r_{i, a_{i}}, a_{i}\right\}$. 
\end{proposition}
\begin{proof}
Recall the definition of \Cref{alg:fairgreedy_2}. For any $a \in [K]$ let $\hat{r}_{t,a} = \mu_{\tz}^{\top}X_{t,a}$, which is the estimated reward for arm $a$, at round $t$. Note that $\mu_{\tz}$ and $X_{t,a}$ are independent random variables. Furthermore, denote with $\cdf_{\hat{r}_{t,a}} (\cdot, s_{t,a})$  the CDF of $\hat{r}_{t,a}$ conditioned on $\mu_{\tz}$ and $s_{t,a}$, and let
\begin{align*}
    \phi_{t,a} := \cdf_{\hat{r}_{t,a}}(\hat{r}_{t,a} , s_{t,a})\enspace, \quad\text{and}\quad \hat \phi_{t,a} := \hat \cdf_t(\hat{r}_{t,a} , s_{t,a})\enspace.
\end{align*}
Let $C_t : = \argmax_{a\in [K]} \hat\phi_{t,a}$. Now, by the definition of the algorithm, we have
\begin{align*}
    \Pr(a_t = a \,|\, \{s_{t,a}\}_{a=1}^K, \Hminus_{t-1} ) &= \sum_{m=1}^K \frac{1}{m} \Pr(a \in C_t, |C_t| = m \,|\, \Hminus_{t-1} )\enspace,
\end{align*}
Let $\epsilon_t > 0$ and $\tilde \Pr(\cdot) = \Pr(\cdot \,|\, \{s_{t,a}\}_{a=1}^K, \Hminus_{t-1}, \sup_{a\in[K]}|\phi_{t,a} - \hat \phi_{t,a}| \leq \epsilon_{t} ) $. Then, we can write
\begin{align*}
   \tilde \Pr(a_t = a ) 
   &\geq  \tilde \Pr(\hat \phi_{t,a} > \hat\phi_{t,a'} \,,\,\forall a'\neq a ) \geq \tilde \Pr(\phi_{t,a}  > \phi_{t,a'} + 2 \epsilon_{t}\,,\,\forall~a'\neq a)\enspace,
\end{align*}
where in the first inequality we considered the case when $a \in C_t$ and $|C_t| = 1$. In the second inequality we considered the worst case scenario where $\hat \phi_{t,a} = \phi_{t,a} -\epsilon_{t}$ and $\hat \phi_{t,a'} = \phi_{t,a'} + \epsilon_{t}$. 
Recall that by the construction of the algorithm $\mu_{\tz} = V_{\tz}^{-1}X_{1:\tz}^\top r_{1:\tz} + (1/\sqrt{d\tz})\cdot\gamma_{\tz}$. for all $i \in [G]$, the additive noise $(1/\sqrt{d\tz})\gamma_{\tz}$ assures that $\mu_{\tz}^\top B_i \neq 0$, almost surely. 
Therefore, by \Cref{lm:prodabs} $\hat r_{t,a} = \scal{\mu_{\tz}}{X_{t,a}}$ conditioned on $\mu_{\tz}$ is absolutely continuous.

\cref{assump}\ref{ass:indep} and \cite[Theorem 2.1.10]{casella2021statistical} yield that $\{\phi_{t,a}\}_{a\in [K]}$ conditioned to $\{s_{t,a}\}_{a=1}^K$ are independent and uniformly distributed on $[0,1]$ and in turn that
\begin{align}\label{eq1:prbound_2}
  \tilde\Pr(a_t = a  )&\geq \int_{0}^{1} \hspace{-.05truecm}\left(\Pr(\phi_{t,a'} < \mu \hspace{-.05truecm}-\hspace{-.05truecm} 2 \epsilon_{t} )\right)^{K-1}\d\mu =\int_{2\epsilon_t}^{1}\hspace{-.05truecm}(\mu{-}2\epsilon_{t})^{K-1}\d\mu =  \frac{(1\hspace{-.05truecm}-\hspace{-.05truecm}2\epsilon_{t})^K}{K}.
\end{align}
We continue by computing an $\epsilon_t$ for which $\sup_{a\in[K]}|\phi_{t,a} - \hat \phi_{t,a}| \leq \epsilon_{t}$ holds with high probability. Observing that, conditioned on $\mu_{\tz}$ and $\{s_{t,a}\}_{a=1}^K$, $\hcdf_{t, a}(\cdot, s_{t,a})$ is the empirical  CDF of $\cdf_{\hat r_{t,a}}(, s_{t,a})$, we can use \Cref{lm:chernoff} and the Dvoretzky–Kiefer–Wolfowitz-Massart inequality to obtain, for any $a \in [K]$, $t\geq t_N$, and $s\geq 0$
\begin{equation*}
    \Pr\left(|\phi_{t,a}- \hat \phi_{t,a}|\geq s\right) \leq 2\exp\left(-2s^2(t-\tz-1)(\alpha q_K)\right)\enspace.
\end{equation*}
Now, let $\tau_0 : = \max\big(t_N, 3 + 8\log^{3/2}\big(5G\e/\delta\big)\big(1- \sqrt[K]{c}\big)^{-3}(\alpha q_K)^{-3/2}\big)$. By applying the union bound and noticing that we have max of $G$ CDFs and approximate CDFs, we can write
\begin{align*}
    \Pr\left(\sup_{t\geq \tau_0, a\in[K]}|\phi_{t,a}- \hat \phi_{t,a}|\geq s\right) 
    &\leq G\sum_{t=\tau_0}^{\infty}\Pr\left(|\phi_{t,a}- \hat \phi_{t,a}| \geq s\right) \\ 
    &\leq 2G\sum_{t=\tau_0}^{\infty}\exp\left(-2s^2(t-\tz-1)(\alpha q_K)\right).
   \end{align*}
Since $\tz = \floor{\frac{t-1}{2}}$, it is straightforward to check that
\begin{align*}
    \Pr\left(\sup_{t\geq \tau_0, a\in[K]}|\phi_{t,a}- \hat \phi_{t,a}|\geq s\right)  &\leq 2G\int_{t=\tau_0-1}^{\infty}\exp\left(-s^2\alpha q_K t\right)\d t \\
    &\leq \frac{2G}{\alpha q_Ks^{2}}\exp\left(-s^2\alpha q_K(\tau_0-1)\right)\enspace.
\end{align*}
Now, for any $\delta \in (0,1)$, by assigning $s = \sqrt{\frac{\log(4G(\tau_0-1)/\delta)}{(\tau_0-1)\alpha q_K}}$, we get
\begin{align}\label{eq2:prbound_2}
    \Pr\left(\sup_{t\geq \tau_0, a\in[K]}|\phi_{t,a}- \hat \phi_{t,a}|\geq \sqrt{\frac{\log(4G(\tau_0-1)/\delta)}{(\tau_0-1)\alpha q_K}}\right) \le \frac{\delta}{2\log\left(4G(\tau_0-1)/\delta\right)}\leq \frac{\delta}{4}\enspace,
\end{align}
where from $\tau_0 \geq 3, \delta < 1 \implies 4G(\tau_0-1)/\delta \geq 8 \geq e^2 \implies \log\left(4G(\tau_0-1)/\delta\right)\geq 2$ we obtain the last inequality. From \eqref{eq1:prbound_2}, it follows that
\begin{align*}
    \inf_{t\geq \tau_0,a\in[K]}\Pr\left(a_t = a |\{s_{t,a}\}_{a=1}^K, \Hminus_{t-1}\right)\geq \frac{(1-2\sup_{t\geq \tau}\epsilon_{t})^{K}}{K}\enspace.
\end{align*}
Moreover, form \eqref{eq2:prbound_2}, by letting $\epsilon_t = \sqrt{\frac{\log(4G(\tau_0-1)/\delta)}{(\tau_0-1)\alpha q_K}}$, with probability at least $1-\frac{\delta}{4}$, we have
\begin{align}\label{eq3:prbound_2}
    \inf_{t\geq \tau, a\in[K]}\Pr\left(a_t = a |\{s_{t,a}\}_{a=1}^K, \Hminus_{t-1}\right)\geq \frac{1}{K}\left(1-2\sqrt{\underbrace{\frac{\log(4G(\tau_0-1)/\delta)}{(\tau_0-1)\alpha q_K}}_{\text{(I)}}}\right)^{K}\enspace.
\end{align}
For the term (I) in the above, using $\log(x) \leq \log(5\e/4)x^{1/3}$ and $x \geq x^{2/3}$ for any $x \geq 1$ we deduce that
\begin{align*}
    \text{(I)}  = \frac{\log(4G/\delta) + \log(\tau_0-1)}{(\tau_0-1)\alpha q_K} \leq  \frac{\log(4G/\delta) + \log(5\e/4)}{(\tau_0-1)^{2/3}\alpha q_K} = \frac{\log(5G\e/\delta)}{(\tau_0-1)^{2/3}\alpha q_K} \enspace.
\end{align*}
Now, since $\tau_0 \geq  3 +  8\log^{3/2}\big(5G\e/\delta\big)\big(1- \sqrt[K]{c}\big)^{-3}(\alpha q_K)^{-3/2}$, we get that $\text{(I)} \leq \frac{1}{4}\left(1-\sqrt[K]{c}\right)^{2}$ and conclude the proof by plugging this inequality in \eqref{eq3:prbound_2}.
\end{proof}

Furthermore, for fixed $t$, let $\tilde \E =  \E [\cdot \,|\, \Hminus_{t-1}]$ and $\tilde \Pr =  \Pr [\cdot \,|\, \Hminus_{t-1}]$. Note that if the assumptions of \Cref{prop:patlowerbound_2} are satisfied, then
\begin{align}
    \nonumber \tilde \E [X_{t,a_t} X_{t,a_t}^\top ] &= \sum_{i=1}^G  \E [\hat X_i \hat X_i^\top ] \tilde\Pr (s_{t, a_t} = i) \\
    \nonumber &=    \sum_{i=1}^G  \E [\hat X_i \hat X_i^\top ] \sum_{a=1}^K \tilde \Pr (a_t = a\,|\,s_{t,a} = i)  \tilde \Pr (s_{t, a} = i) \\
    \label{eq:patboundused}&\geq c K^{-1} \sum_{i=1}^G  \E [\hat X_i \hat X_i^\top ] q_K = c \frac{q_K G}{K} \frac{1}{G}\sum_{i=1}^G  \E [\hat X_i \hat X_i^\top ] 
\end{align}
where we applied \Cref{prop:patlowerbound_2} in the last line.
We can bound $\norm{x_{\max}}_{V_{\tz}^{-1}}$ in \Cref{lm:instantbound_2} in the same way as in \Cref{lm:boundx} using \eqref{eq:patboundused}  with $c=1/2$ when needed in the proof of \Cref{lm:eiguen}. Combining the previous results we obtain the following regret bound.

\begin{theorem}\label{th:regret_2}
Let \Cref{ass:b} hold, $a_t$ be generated by \Cref{alg:fairgreedy_2} and $\Sigma : = G^{-1} \sum_{i=1}^G  \E [\hat X_i \hat X_i^\top ] $ Then, with probability at least $1-\delta$, for any $T\geq 1$ we have
\begin{align*}
     \reg(T) \leq&\frac{96M\xmax\sqrt{K}}{\sqrt{\eigminplus(\Sigma)q_KG}}\left[(\lambda^{\frac{1}{2}}+R + \xmax)\sqrt{dT\log((8+4T\max(\xmax^2/\lambda,1))/\delta_1)} +\sqrt{\lambda T}\norm{\mu^*}_2\right] \\
     & \quad+8\sqrt{\frac{T\log(8GT/\delta_1)}{3\alpha q_K}} + \tau\enspace,
\end{align*}
where $\delta = \delta_1 + \isrand\delta_2$, $\tau = 4\max(\tau_1, \tau_2, \isrand\tau_3) +3$ and
\begin{align*}
    \tau_1 = \frac{32K^3}{(\alpha q_K)^{3/2}}\log^{3/2}\big(5G\e/\delta_1\big),  \quad
    \tau_2 = \frac{54\xmax^2}{\eigminplus(\Sigma)}\log(4d /\delta_1), \quad
    \tau_3 = \frac{2}{(1-\alpha)^2 q_K}\log(GT/\delta_2),
\end{align*}
where $q_K$, $\isrand$ and  $\alpha$ are defined in \Cref{lm:chernoff} and we use the convention $\isrand \tau_3 = 0$ if $\isrand=0, \tau_3 = \infty$. Hence

\begin{equation*}
    \reg(T) = O \left(\isrand \frac{\log(GT/\delta_2)}{q_K} + \frac{K^3\log^{3/2}(G/\delta_1)}{q_K^{3/2}} + \sqrt{\frac{dT\log\left(GT/\delta_1\right)}{(1+ G/K)q_K)}}  \right)\enspace.
\end{equation*}

\end{theorem}
\begin{proof}[Proof Sketch]
    First, assume $S_{T}(t_N, \alpha)$ holds and use a similar strategy of \Cref{th:regret} to get a bound w.p. at least $1-\delta_1$. Then combine this result with  \Cref{lm:chernoff}.
\end{proof}
Notice that in the case where each arm corresponds to a different sensitive group, i.e.\@ when $G = K$,  $s_a = a$ and therefore $q_K = 1$, $\isrand = 0$ and $\alpha = 1$, we recover \Cref{th:regret}. Moreover, we have the following corollary which shows an advantage for higher number of arms compared to the bound in
\Cref{th:regret} when $\{(X_a, s_a)\}_{a=1}^K$ are i.i.d..

\begin{corollary}\label{corr:equalprob}
    Let $\{(X_a, s_a)\}_{a=1}^K$ be i.i.d. and $q_{\min} : = \min_{i \in [G]} \Pr(s_a = i) G$.
    If \Cref{ass:b} holds and $a_t$ is generated via \Cref{alg:fairgreedy_2} we have that with probability at least $1-\delta$, for any $T\geq 1$ and $\alpha \in (0, 1)$ we have that
\begin{align*}
     \reg(T) \leq&\frac{96M\xmax}{\sqrt{\eigminplus(\Sigma)q_{\min}}}\left[(\lambda^{\frac{1}{2}}+R + \xmax)\sqrt{dT\log((8+4T\max(\xmax^2/\lambda,1))/\delta_1)} +\sqrt{\lambda T}\norm{\mu^*}_2\right] \\
     & \quad+8\sqrt{\frac{TG\log(8GT/\delta_1)}{3K\alpha q_{\min} }} + \tau\enspace,
\end{align*}
where $\delta = \delta_1 + \delta_2$, $\tau = 4\max(\tau_1, \tau_2, \tau_3) +3$ and
\begin{align*}
    \tau_1 = \frac{32(KG)^{3/2}}{(\alpha q_{\min})^{3/2}}\log^{3/2}\left(\frac{5G\e}{\delta_1}\right), \,
    \tau_2 = \frac{54\xmax^2}{\eigminplus(\Sigma)}\log\left(\frac{4d}{\delta_1}\right),  \,
    \tau_3 = \frac{2G}{(1-\alpha)^2 K q_{\min}}\log\left(\frac{GT}{\delta_2}\right).
\end{align*}
Hence
\begin{equation*}
    \reg(T) = O \left( \frac{G\log(GT/\delta)}{K q_{\min}} + \frac{(K G)^{3/2}\log^{3/2}(G/\delta)}{q_{\min}^{3/2}} + \sqrt{\frac{dT\log\left(GT/\delta\right)}{(1 + K/G)q_{\min})}}  \right)
\end{equation*}
\end{corollary}
Note that in \Cref{corr:equalprob}, $q_{\min} > 0$ without loss of generality and $q_{\min} =1$ if and only if each group has the same probability of being sampled. Furthermore $q_{\min}/G$ is the probability that a context belongs to the less common group, which depends on the problem at hand. Note that there is an advantage compared to \Cref{th:regret} in terms of number of arms  when $K > G$. This is because context coming from all arms can be use to estimate the CDF of a given group.

\subsection{Additional Details on the US Census Experiments}\label{se:expmultigroup}

\begin{figure}[t!]
     \centering
     \begin{subfigure}[b]{0.49\textwidth}
         \centering
\includegraphics[width=\textwidth]{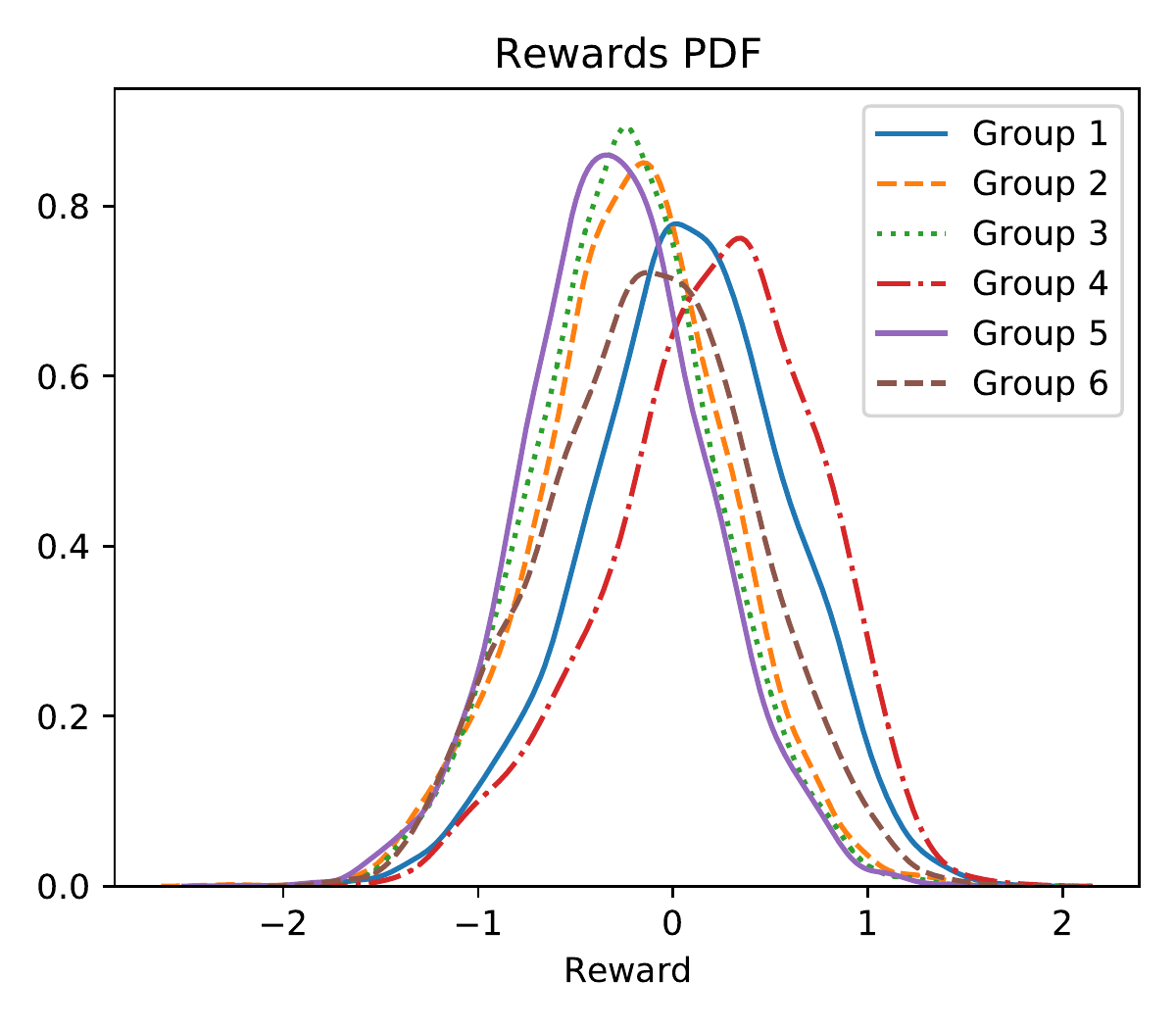}
     \end{subfigure}
     \hfill
     \begin{subfigure}[b]{0.49\textwidth}
         \centering
         \includegraphics[width=\textwidth]{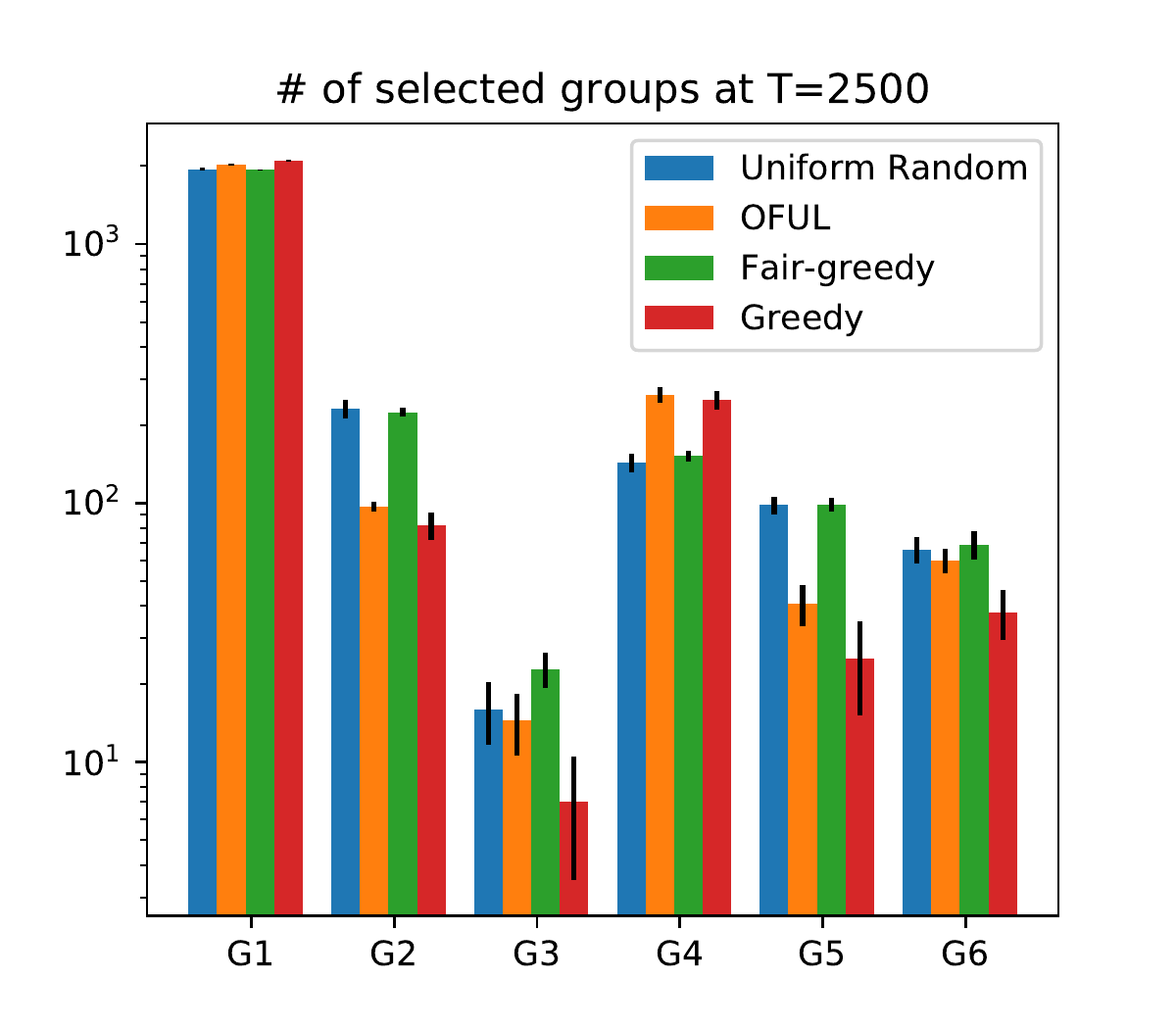}
     \end{subfigure}
     \hfill
    \begin{subfigure}[b]{0.49\textwidth}
    \centering
    \includegraphics[width=\textwidth]{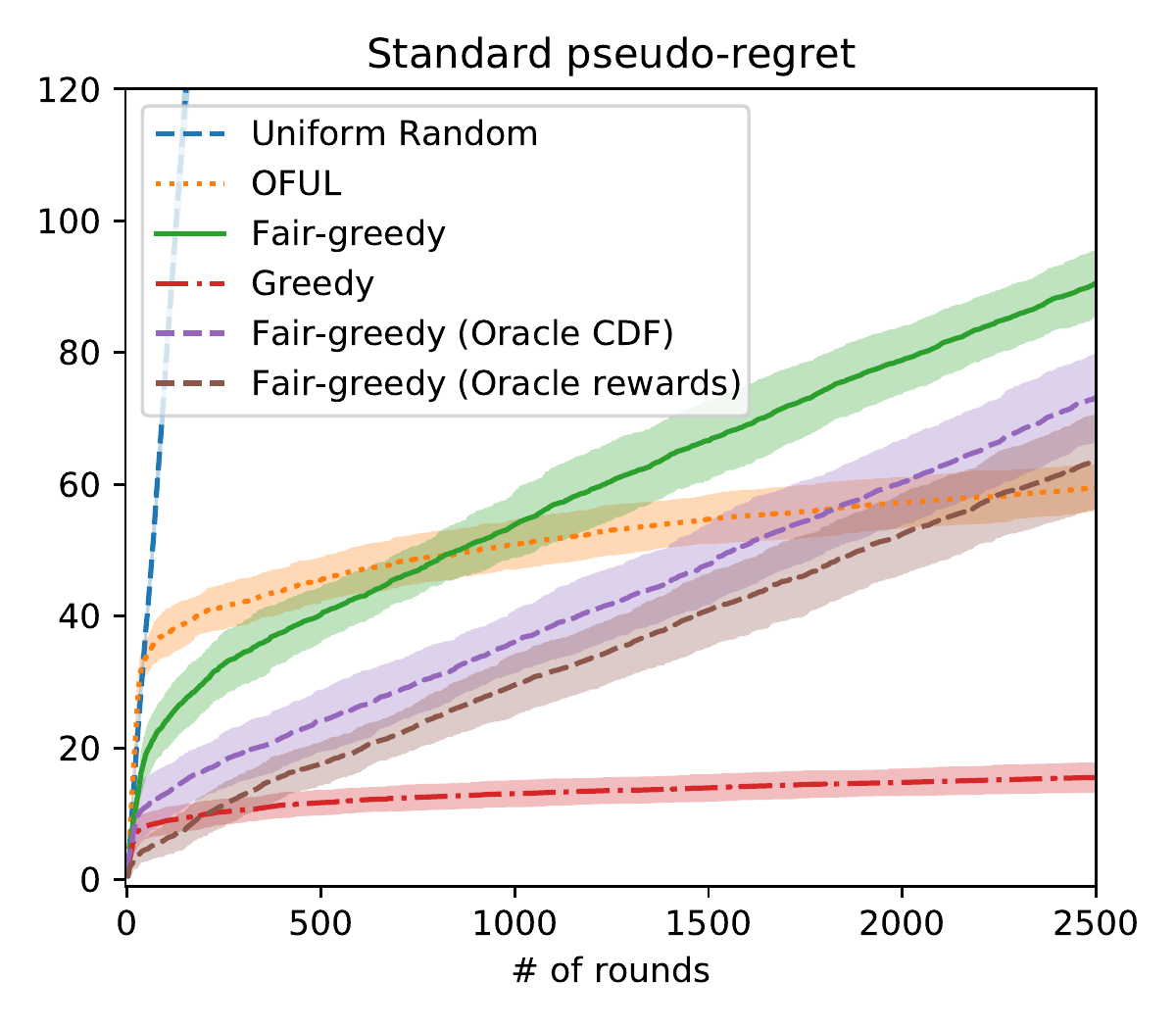}
     \end{subfigure}
     \hfill
     \begin{subfigure}[b]{0.49\textwidth}
         \centering
         \includegraphics[width=\textwidth]{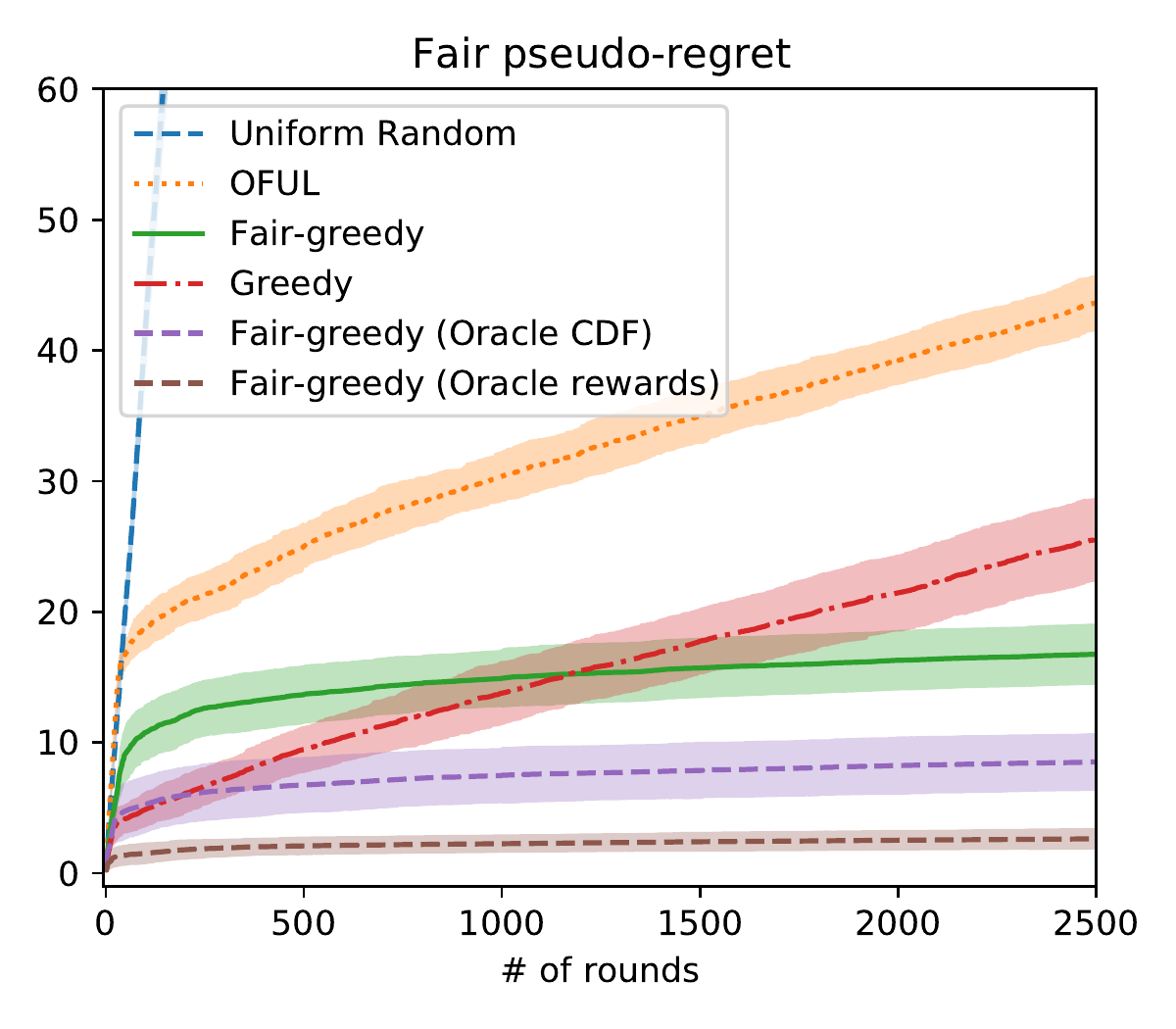}
     \end{subfigure}
        \caption{\small \textbf{US Census Results. Group $=$ Ethnicity}. 
        First image is the density plots of the reward distributions, the second image is the number candidates (in log scale) from each group which are selected by each policy (mean and std over 10 runs), while the bottom two plots are the standard and fair pseudo-regrets, with mean (solid lines) $\pm$ standard deviation (shaded region) over 10 runs. To compute the reward CDF for each group we use the empirical CDF on $5K$ samples from $D2$.
        }
        \label{fig:expcensus_2}
\end{figure}

This experiment is introduced in \Cref{se:multiplemain,} and similarly to that of \Cref{se:censusone}, is performed using the US Census data. However, candidates are sampled from the original dataset at random together with their sensitive group (their ethnicity). Hence, we use Fair-greedy V2 (\Cref{alg:fairgreedy_2}). Differently from \Cref{se:censusone} where we use the target income as noisy reward, here we add artificial noise with standard deviation $0.2$ directly to the true reward.

\textbf{Setup and Preprocessing.} To setup the bandit problem, we construct two datasets: $D1$ and $D2$. We first load all the data from the $2017$ US Census Survey to assemble $D1$, and then from the $2018$ survey to assemble $D2$. Then we retain only candidates from $6$ ethnic groups containing at least $5 \times 10^3$ candidates, in order to accurately compute the true CDF for each group. We use $D1$ to find mean and standard deviation for each feature and also for the target. After that we normalize features and target of $D2$ by subtracting the mean and dividing by the standard deviation previously computed on $D1$. We then construct $\mu^*$ as a ridge regression estimate on the samples from $D2$ with the regularization parameter equal to $10^{-8}$. The regression vector $\mu^*$ will be used to compute the (true) rewards for the samples. We construct the bandit problem as follows. At each round, the context vectors of $K=10$ individual are sampled from $D2$ and after one is selected by the policy, its corresponding noisy reward, obtained by adding gaussian noise with standard deviation $0.2$ to the true reward, is received by the agent. 

\textbf{Baselines.} We compare our method with the same baselines as in \Cref{se:censusone}, where the two oracle policies are now variants of Fair-Greedy V2. Moreover, we set the regularization parameter for all policies using a ridge estimate to 0.1 and the exploration parameter of OFUL to 0.01.

\textbf{Results (\Cref{fig:expcensus_2}).} We draw similar conclusions as in \Cref{se:censusone}. In particular, Greedy performing better than OFUL and the Fair-Greedy policy achieving sublinear fair pseudo-regret, but worse than Oracle methods. Additionaly we can see that knowing $\mu^*$ plays a more important role than knowing the true reward CDFs. In this case, the gap between the Uniform random policy and the others is even larger since $K=10$. 
Moreover, as expected, Fair-greedy selects more candidates from underperforming (in terms of reward) minority groups, when compared with OFUL and Greedy.

\section{Trade off between Fairness and Reward Maximization}\label{sec:linreg}

In this section, we show for which problems the GMF policy and the optimal policy have competing goals. in particular, for the case of $K=2$, when the rewards are absolutely continuous and independent across arms, whenever they are not identically distributed, the GMF policy achieves linear standard pseudo-regret with nonzero probability. The following theorem proves this result.

\begin{theorem}\label{thm:linreg_opt}
Let \Cref{assump} hold with $K=2$, and assume that $\mathcal{F}_1\neq \mathcal{F}_2$. Let $\bar{r}_{t,a} : =\scal{\mu^*}{X_{t,a}}$,  $\{a_t^*\}_{t=1}^T$ be the GMF policy (see \Cref{def:gmf}) and $\{a_t^\text{opt}\}_{t=1}^T$ be the optimal policy (see \Cref{rem:standard}).
Then, there exists $\epsilon>0$, such that 
\begin{equation*}
p = \int_0^1\left[\max(\cdf_1(\cdf_2^{-1}(y) -\epsilon)- y, 0) + \max(y- \cdf_1(\cdf_2^{-1}(y)+\epsilon), 0) \right] \d y > 0\enspace.
\end{equation*}
Furthermore with probability at least $\frac{\epsilon p}{4L\norm{\mu^{*}}}$, for any $T>0$, we have
\begin{align}\label{Eq1:linreg}
    T\cdot\frac{\epsilon p}{2}\leq \sum_{t=1}^{T}\big[\bar{r}_{t,a_t^{\text{opt}}} - \bar{r}_{t,a_{t}^*}\big]\enspace.
\end{align}
\end{theorem}
\begin{proof}
Let $\bar{r}_{a} : =\scal{\mu^*}{X_{a}}$, $q_a = \cdf_a(r_a)$ be the CDF value of $r_a$ and $\cdf_a^{-1}$ be the quantile function, i.e.\@ such that $\cdf_a^{-1}(x) = \inf \{y \in \R : x \leq \cdf_a(y) \}$. For $\epsilon>0$ consider the set $E^\epsilon := E^\epsilon_1 \cup E^\epsilon_2$ where
\begin{align*}
E^\epsilon_1 &:= \{(x, y) \in [0,1]^2 : x > y, \cdf_1^{-1}(x) < \cdf_2^{-1}(y)-\epsilon \}\enspace, \\
E^\epsilon_2 &:= \{(x, y) \in [0,1]^2 : x < y, \cdf_1^{-1}(x)  > \cdf_2^{-1}(y) +\epsilon \}\enspace.
\end{align*}
Note that we can write
\begin{align*}
E^\epsilon_1 &= \{(x, y) \in [0,1]^2 : y < x <  \cdf_1(\cdf_2^{-1}(y)-\epsilon)\}\enspace, \\
E^\epsilon_2 &= \{(x, y) \in [0,1]^2 : \cdf_1(\cdf_2^{-1}(y)+\epsilon) < x < y \}\enspace.
\end{align*}
Now, let $g_{1,2}(y, \epsilon) = \cdf_1(\cdf_2^{-1}(y)+\epsilon)$. Since from \Cref{assump}\ref{ass:iid}\ref{ass:abscont}, $q_1$ an $q_2$ are $i.i.d$ uniform on $[0,1]$ we have that
\begin{equation}\label{eq:last}
\begin{aligned}
    \Pr((q_1, q_2) \in E^\epsilon) &=\Pr( (q_1, q_2) \in E^\epsilon_1) + \Pr( (q_1, q_2) \in E^\epsilon_2) \\
    &=\int_0^1\int_{y}^{g_{1,2}(y, -\epsilon)} \text{d}x \text{d}y + \int_0^1\int_{g_{1,2}(y, \epsilon)}^{y} \text{d}x \text{d}y \\
    &=\int_0^1\max(g_{1,2}(y, -\epsilon)- y, 0) \text{d}y + \int_0^1\max(y- g^\epsilon_{1,2}(y, \epsilon), 0)  \text{d}y \\
    &=\int_0^1\left[\max(\cdf_1(\cdf_2^{-1}(y) -\epsilon)- y, 0) + \max(y- \cdf_1(\cdf_2^{-1}(y)+\epsilon), 0) \right] \text{d}y\enspace.
\end{aligned}
\end{equation}
Since $\mathcal{F}_1 \neq \mathcal{F}_2$, and $\mathcal{F}_1, \mathcal{F}_2$ are absolutely continuous, there exists $\epsilon'>0$, such that 
 $\mathcal{F}_2^{-1}(y) - \mathcal{F}_1^{-1}(y)>\epsilon'$, or $\mathcal{F}_2^{-1}(y) - \mathcal{F}_1^{-1}(y)<\epsilon'$ for y inside a closed interval, and hence 
$\Pr((q_1, q_2) \in E^{\epsilon'}) > 0$. 
This yields \Cref{Eq1:linreg} by letting $\epsilon = \epsilon'$, and $p = \Pr((q_1, q_2) \in E^{\epsilon})$. 

Now, let $q_{t,a} = \cdf_a(\bar{r}_{t,a})$, then for the expected value of the instantaneous standard regret, at round $t$, we can write 
\begin{align*}
    \E\big[\bar{r}_{t,a_t^{\text{opt}}} - \bar{r}_{t,a_{t}^*}\big] &\geq 
    \int_{(x,y) \in E^{\epsilon}} \abs{\cdf_{2}^{-1}(y) - \cdf_1^{-1}(x)} \d x \d y \geq \epsilon \Pr((q_{t,1}, q_{t,2}) \in E^{\epsilon}) = \epsilon p>0\enspace,
\end{align*}
and for the standard regret, we have
\begin{align*}
    \sum_{t=1}^{T}\E\big[\bar{r}_{t,a_t^{\text{opt}}} - \bar{r}_{t,a_{t}^*}\big] \geq  T\cdot \epsilon p>0\enspace.
\end{align*}
Finally, let $\Omega$ be the event that $ \frac{1}{2}\cdot\sum_{t=1}^{T}\E\big[\bar{r}_{t,a_t^{\text{opt}}} - \bar{r}_{t,a_{t}^*}\big]\leq\sum_{t=1}^{T}[\bar{r}_{t,a_t^{\text{opt}}} - \bar{r}_{t,a_{t}^*}]$. Considering the fact that $\sum_{t=1}^{T}[\bar{r}_{t,a_t^{\text{opt}}} - \bar{r}_{t,a_{t}^*}] \leq 2L\norm{\mu^{*}}T$, we deduce
\begin{align*}
\sum_{t=1}^{T}\E[r_{a_t^{\text{opt}}} - \bar{r}_{t,a_{t}^*}] &= \sum_{t=1}^{T}\left[ \E[ \bar{r}_{t,a_t^{\text{opt}}} - \bar{r}_{t,a_{t}^*}\,|\, \Omega]\Pr(\Omega) + \E[ \bar{r}_{t,a_t^{\text{opt}}} - \bar{r}_{t,a_{t}^*}\,|\, \Omega^c] \Pr(\Omega^c)\right] \\ &\leq 2L\norm{\mu^{*}}T\Pr(\Omega) + \sum_{t=1}^{T}\E[\bar{r}_{t,a_t^{\text{opt}}} - \bar{r}_{t,a_{t}^*}]/2 \enspace,
\end{align*}
and we get $\frac{\epsilon p}{4L\norm{\mu^{*}}}\leq \sum_{t=1}^{T}\E[\bar{r}_{t,a_t^{\text{opt}}} - \bar{r}_{t,a_{t}^*}]/(4L\norm{\mu^{*}}T)\leq \Pr(\Omega)$, which finishes the proof. 
\end{proof}
\begin{remark}
In \Cref{thm:linreg_opt}, $\epsilon \leq 2L\norm{\mu^{*}}$, otherwise $p = 0$. On the other hand, by the definition $p\leq 1$, and accordingly $\frac{\epsilon p}{4L\norm{\mu^{*}}} \leq 1/2$.
\end{remark}
\begin{remark}
With similar reasoning as in the proof of \Cref{thm:linreg_opt}, we can show that if $\mathcal{F}_1 \neq \mathcal{F}_2$ the optimal policy (see \Cref{rem:standard}) has linear fair pseudo-regret with positive probability, that is independent of $T$. In particular, there exist $c, c'>0$, such that for any  $T>0$, $\Pr(T\cdot c' \leq \sum_{t=1}^{T}[
\cdf_{a_t^*} (\bar{r}_{t,a_{t}^*})-\cdf_{a_t^{\text{opt}}} (\bar{r}_{t,a_t^{\text{opt}}}) ])> c$.
\end{remark}

\begin{example}[Disjoint supports]
As an example consider the case when $K = 2$ and $\bar{r}_{t,1} - \bar{r}_{t,2}\geq\epsilon>0 $, for all $t\geq 1$, almost surely. Then, $\cdf_1(\cdf_2^{-1}(y) - \epsilon) = \cdf_1(\cdf_2^{-1}(y) + \epsilon) = 0$ for every $y \in [0,1]$. Hence  we have 
\begin{align*}
p = \int_0^1\left[\max(\cdf_1(\cdf_2^{-1}(y) -\epsilon)- y, 0) + \max(y- \cdf_1(\cdf_2^{-1}(y)+\epsilon), 0) \right] \d y = 1/2\enspace.
\end{align*}
 Then by \Cref{thm:linreg_opt}, with probability at least $\frac{\epsilon}{8L\norm{\mu^*}}$, for any $T>0$, we have $\sum_{t=1}^{T}[\bar{r}_{t,a_t^{\text{opt}}} - \bar{r}_{t,a_{t}^*}]\geq \frac{T\epsilon}{4}$.
\end{example}

\end{document}